\documentclass[twoside]{article}

\usepackage[accepted]{aistats2019}

\usepackage[round]{natbib}

\usepackage[utf8]{inputenc}
\usepackage{xcolor}
\usepackage{definitions}
\usepackage{amsfonts}     
\usepackage{amsmath}
\usepackage{amsthm}
\usepackage{afterpage}
\usepackage[colorlinks]{hyperref}

\hypersetup{
    linkcolor=red,
    citecolor=[rgb]{0,0,0.5},
    filecolor=magenta,      
    urlcolor=blue           
}

\newtheorem{thm}{Theorem}
\newtheorem{cor}[thm]{Corollary}
\newtheorem{prop}[thm]{Proposition}
\newtheorem{lem}[thm]{Lemma}

\theoremstyle{definition}

\theoremstyle{plain}

\usepackage{graphicx}

\newcommand{\Dgen}{\Dcal}

\newcommand{\post}{{\hat\rho}}
\renewcommand{\posterior}{\post}

\renewcommand{\loss}{\ell}

\newcommand{\etamu}{\mu}

\newcommand{\kQt}{\widehat{k}_{Q^l}}
\newcommand{\kQl}{\kQt}

\usepackage{booktabs}
\usepackage{multirow}
\usepackage{multicol}
\usepackage{subcaption}

\usepackage{floatrow}
\newfloatcommand{capbtabbox}{table}[][\FBwidth]

\begin{document}
	
	\runningauthor{Ga\"el Letarte, Emilie Morvant, Pascal Germain}

	\twocolumn[

	\aistatstitle{Pseudo-Bayesian Learning with Kernel Fourier Transform as Prior}

	\aistatsauthor{Ga\"el Letarte$^1$\\gael.letarte.1@ulaval.ca \And Emilie Morvant$^2$\\emilie.morvant@univ-st-etienne.fr \And Pascal Germain$^3$ \\pascal.germain@inria.fr}
	\aistatsaddress{$^1$ D\'epartement d'informatique et de g\'enie logiciel, Universit\'e Laval,	Qu\'ebec, Canada\\
	$^2$ Univ Lyon, UJM-Saint-Etienne, CNRS, Institut d'Optique Graduate School,\\Laboratoire Hubert Curien UMR 5516, Saint-Etienne, France\\
	$^3$ \'Equipe-projet Modal, Inria Lille - Nord Europe,  Villeneuve d'Ascq, France
	}
		]

\begin{abstract}
We revisit~\citet{rahimi-07}'s kernel random Fourier features (RFF) method through the lens of the PAC-Bayesian theory.
While the primary goal of RFF is to approximate a kernel, we look at the Fourier transform as a \emph{prior} distribution over trigonometric hypotheses. It naturally suggests learning a \emph{posterior} on these hypotheses. We derive generalization bounds that are optimized by learning a \emph{pseudo-posterior} obtained from a closed-form expression. Based on this study, we consider two learning strategies: The first one finds a compact landmarks-based representation of the data where each landmark is given by a distribution-tailored similarity measure, while the second one provides a PAC-Bayesian justification to the kernel alignment method of~\citet{SinhaD16}.
\end{abstract}

\section{INTRODUCTION}

Kernel methods~\citep{shawe-taylor-04-book}, such as support vector machines~\citep{boser-92,vapnik-98}, 
map data in a high dimension space in which a linear predictor can solve the learning problem at hand.
The mapping space is not directly computed and the linear predictor is represented implicitly thanks to a kernel function.
This is the powerful kernel trick: the kernel function computes the scalar product between two data points in this high dimension space.
However, kernel methods notoriously suffer from two drawbacks. On the first hand, computing all the scalar products for all the learning samples is costly: $O(n^2)$ for many kernel-based methods, where $n$ is the number of training data point. On the other hand, one has to select a kernel function adapted to the learning problem for the algorithm to succeed.

The first of these drawbacks has motivated the development of approximation methods making kernel methods more scalable, such as Nystr\"om approximation~\citep{williams2001nystrom,drineas2005nystrom} that constructs a low-rank approximation of the Gram matrix\footnote{The Gram matrix is the $n\times n$ matrix constituted by all the kernel values computed on the learning samples.} and is data dependent, or random Fourier features (RFF)~\citep{rahimi-07} that approximates the kernel with random features based on the Fourier transform and is not data dependent
\citep[a comparison between the two approaches have been conducted by][]{nystromVSrff}.
In this paper, we revisit the latter technique.

We start from the observation that a predictor based on kernel Fourier features can be interpreted as a weighted combination of those features according to a data independent distribution defined by the Fourier transform. 
We introduce an original viewpoint, where this distribution is interpreted as a {\it prior distribution} over a space 
of weak hypotheses---each hypothesis being a simple trigonometric function obtained by the Fourier decomposition.
This suggests that one can improve the approximation by adapting this distribution in regards to data points: we aim at learning a {\it posterior distribution}.
By this means, our study proposes strategies to learn a representation to the data.  
While this representation is not as flexible and powerful than the ones that can be learned by deep neural networks~\citep{Goodfellow-16-book}, we think that it is worthwhile to study this strategy to eventually solve the second drawback of kernel methods that currently heavily rely on the kernel choice.
This in mind, while the majority of work related to random Fourier features focus on the study and improvement of the kernel approximation, we propose here a reinterpretation 
in the light of the PAC-Bayesian theory~\citep{mcallester-99,catoni-07}. We derive generalization bounds that can be straightforwardly optimized by learning a \emph{pseudo-posterior} thanks to a closed-form expression.

The rest of the paper is organized as follows.
Section~\ref{sec:rffsetting} recalls the RFF setting. 
Section~\ref{sec:rffprior} expresses the Fourier transform as a prior leading {\it (i)} to a first PAC-Bayesian analysis and a landmarks-based algorithm in Section~\ref{sec:landmark}, {\it (ii)} to another PAC-Bayesian analysis in Section~\ref{sec:revisited} allowing to justify the kernel alignment method of~\citet{SinhaD16} and to propose a greedy kernel learning method.
Then Section~\ref{sec:experiments} provides experiments to show the usefulness of our work.

\section{RANDOM FOURIER FEATURES}
\label{sec:rffsetting}
\paragraph{Problem setting.}
Consider a classification problem where we want to learn a predictor $f:\Rbb^d\to Y$, from a \mbox{$d$-dimensional} space to a discrete output space (\eg, $Y=\{0,1,\ldots,|Y|{-}1\})$. 
The learning algorithm is given a training set $S=\{(\xbf_i, y_i)\}_{i=1}^n \sim \Dgen^n$  of $n$ \iid\ samples, where $\Dgen$ denotes the data generating distribution over $\Rbb^d\times Y$. 
We consider a positive-semidefinite (PSD) kernel $k:\Rbb^d\times\Rbb^d\to[-1,1]$. 
Kernel machines learn predictors of the form
\begin{equation} \label{eq:f_alpha_k}
f(\xbf) \ = \ \sum_{i=1}^n \alpha_i k(\xbf_i, \xbf)\,,
\end{equation}
by optimizing the values of vector 
$\alphabf \in \Rbb^n$.

\paragraph{Fourier features.}
When $n$ is large, running a kernel machine algorithm (like SVM or kernel ridge regression) is expensive in memory and running time.
To circumvent this problem, \citet{rahimi-07} introduced the \emph{random Fourier features} as a way to approximate the value of a \emph{shift-invariant kernel}, \ie, relying on the value of $\deltabf = \xbf-\xbf'\in \Rbb^d$, which we write 
$$k(\deltabf) = k(\xbf-\xbf') = k(\xbf,\xbf')$$ 
interchangeably.  
Let the distribution $p(\omegabf)$ be the Fourier transform of the shift-invariant kernel $k$, 
\begin{equation}\label{eq:p_omega}
p(\omegabf) \ =\ \frac1{(2\pi)^d}\int_{\Rbb^d} k(\deltabf)\, e^{- i \,\omegabf\cdot\deltabf} d\,\deltabf\,.
\end{equation}
Now, by writing $k$ as the inverse of the Fourier transform $p$, and using trigonometric identities, we obtain:
\begin{align}
k(\xbf - \xbf')
&= \nonumber
\int_{\Rbb^d} p(\omegabf) e^{i\,\omegabf\cdot(\xbf-\xbf')} d\,\omegabf 
\ = \ 
\Esp_{\omegabf\sim p} e^{i\,\omegabf\cdot(\xbf-\xbf')} \\
&=  \nonumber
\Esp_{\omegabf\sim p} \Big[\cos\big(\omegabf\cdot(\xbf{-}\xbf')\big) {+} i \sin\big(\omegabf\cdot(\xbf{-}\xbf')\big)\Big]\\
&= \label{eq:kernel_to_cos}
\Esp_{\omegabf\sim p} \cos\big(\omegabf\cdot(\xbf-\xbf')\big)\,.
\end{align}
\citet{rahimi-07} suggest expressing the above  
$\cos\!\big(\omegabf\cdot(\xbf{-} \xbf')\big)$
as a product of two features. 
One way to achieve this is to map every input example into 
\begin{equation} \label{eq:zcossin}
\zbf_\omegabf(\xbf) \ =\   \big( \cos(\omegabf\cdot\xbf), \ \sin (\omegabf\cdot\xbf)\big) \,.
\end{equation}
The random variable \mbox{$\zbf_\omegabf(\xbf) \cdot \zbf_\omegabf(\xbf')$}, with $\omegabf$ drawn from~$p$, is an unbiased estimate of $k(\xbf-\xbf')$. Indeed, we recover from Equation~\eqref{eq:kernel_to_cos} and Equation~\eqref{eq:zcossin}:
\begin{align*}
	\Esp_{\omegabf\sim p} & \zbf_\omegabf(\xbf) \cdot \zbf_\omegabf(\xbf')\
	\\[-1mm]
	 =&  
	\Esp_{\omegabf\sim p} \Big[\cos (\omegabf\cdot\xbf)  \cos (\omegabf\cdot\xbf') + \sin (\omegabf\cdot\xbf)  \sin (\omegabf\cdot\xbf')\Big] \\
	 =&  
	\Esp_{\omegabf\sim p} \cos\big(\omegabf\cdot(\xbf-\xbf')\big)\,.
\end{align*}
To reduce the variance in the estimation of \mbox{$k(\xbf{-} \xbf')$}, the idea is to sample $D$ points \iid\ from $p$: $\omegabf_1,\omegabf_2,\ldots, \omegabf_D$. Then, each training sample $\xbf_i\in\Rbb^d$ is mapped to a new feature vector in $\Rbb^{2D}$\,:
\begin{align}\label{eq:phibf_mapping}
\phibf(\xbf_i)  =  \frac1{\sqrt D} \Big( \cos&(\omegabf_1\cdot \xbf_i)\,,\ \ldots\,,\ \cos(\omegabf_D\cdot \xbf_i) \,,
\\[-2mm]  &\sin(\omegabf_1\cdot \xbf_i)\,,\ \ldots\,,\ \sin(\omegabf_D\cdot \xbf_i) \Big)\,.\nonumber
\end{align} 
Thus, we have 
$k(\xbf - \xbf') \approx \phibf(\xbf) \cdot \phibf(\xbf')\,$ when $D$ is ``large enough''.
This provides a decomposition of the PSD kernel $k$ that differs from the classical one~\citep[as discussed in][]{bach-17-equivalence}. By learning a linear predictor on the transformed training set \mbox{$S\mapsto\{(\phibf(\xbf_i), y_i)\}_{i=1}^n$} through an algorithm like a linear SVM, we recover a predictor equivalent to the one learned by a kernelized algorithm. 
That is, we learn a weight vector \mbox{$\wbf =  (w_1,\ldots,w_{2D}) \in \Rbb^{2D}$} and we predict the label of a sample $\xbf\in\Rbb^d$ by computing, in place of Equation~\eqref{eq:f_alpha_k},
\begin{equation} \label{eq:f_w_phi}
f(\xbf) \ = \ \sum_{j=1}^{2D} w_j \,\phi_j(\xbf)\,.
\end{equation}

\section{THE FOURIER TRANSFORM AS A PRIOR DISTRIBUTION}
\label{sec:rffprior}
As described in the previous section, the random Fourier features trick has been introduced to reduce the running time of kernel learning algorithms. 
Consequently, most of the subsequent work study and/or improve the properties of the kernel approximation \citep[\eg,][]{Yu16,Rudi17,bach-17-equivalence,Choromanski18} with some notable exceptions, 
as the \emph{kernel learning} algorithms of \cite{yang2015carte}, \citet{SinhaD16}, and \citet{oliva2016bayesian}, that we discuss and relate to our approach in Section~\ref{sec:revisited}.

We aim at reinterpreting the Fourier transform---\ie, the distribution $p$ of Equation~\eqref{eq:p_omega}---as a \emph{prior distribution} over the feature space.
It can be seen as an alternative representation of the prior knowledge that is encoded in the choice of a specific kernel function, that we denote $k_p$ from now on. 
In accordance with Equation~\eqref{eq:kernel_to_cos}, each feature obtained from a vector $\omegabf\in\Rbb^d$ can be seen as a hypothesis
$$h_\omegabf(\deltabf) \eqdef \cos(\omegabf \cdot \deltabf)\,.$$
Henceforth, the kernel is interpreted as a predictor performing a \mbox{$p$-weighed} aggregation of weak hypotheses.
This alternative interpretation of distribution $p$ as a prior over hypotheses naturally suggests to \emph{learn a posterior distribution} over the same hypotheses.
That is, we seek a distribution $q$ giving rise to a new kernel $$k_q(\deltabf) \eqdef \Esp_{\omegabf\sim q} h_\omegabf(\deltabf)\,.$$
In order to assess the quality of the kernel $k_q$, we define a loss function based on the consideration that its output should be high when two samples share the same label, and low otherwise. 
Hence, we evaluate the kernel on two samples \mbox{$(\xbf, y)\sim \Dgen$} and \mbox{$(\xbf', y')\sim \Dgen$} through the linear loss
\begin{equation} \label{eq:linloss}
\ell\big(k_q(\deltabf), \lambda\big)\ \eqdef \ 
\frac{1 - \lambda \,k_q(\deltabf)}2 \,,
\end{equation}
where $\deltabf= \xbf {-} \xbf'$ denotes a pairwise distance and $\lambda$ denotes the pairwise similarity measure:
$$\lambda=\lambda(y,y')\eqdots\left\{\begin{array}{rl} 1 & \mbox{ if $y=y'$,}\\
-1 & \mbox{ otherwise.}\end{array}\right.$$
Furthermore, we define the \emph{kernel alignment} generalization loss $\Lcal_\Delta(k_q) $ on a ``pairwise'' probability distribution $\Delta$, defined over $\Rbb^d{\times}[-1,1]$ as
\begin{equation} \label{eq:genloss1}
\Lcal_\Delta(k_q) \,\eqdef \Esp_{(\deltabf,\lambda)\sim \Delta}
\ell\big(k_q(\deltabf), \lambda\big)\,.
\end{equation}
Note that any data generating distribution $\Dgen$ over \mbox{input-output} spaces $\Rbb^d\times Y$ automatically gives rise to a ``pairwise'' distribution $\Delta_\Dcal$.  By a slight abuse of notation, we write $\Lcal_\Dcal(k_q)$ the corresponding generalization loss, and the associated  \emph{kernel alignment} empirical loss is defined as 
\begin{equation} \label{eq:emploss1}
\widehat\Lcal_{S}(k_q) \eqdef 
\frac{1}{n(n-1)}
\ \sum_{i,j=1, i\ne j}^{n}
\ \ell\big(k_q(\deltabf_{ij}), \lambda_{ij}\big)\,,
\end{equation}
where for a pair of examples \mbox{$\left\{(\xbf_i,y_i),(\xbf_j,y_j)\right\}\in S^2$} we have \mbox{$\deltabf_{ij}\eqdef(\xbf_i-\xbf_j)$} and \mbox{$\lambda_{ij}\eqdef \lambda(y_i,y_j)$}.

Starting from this reinterpretation of the Fourier transform, we provide in the rest of the paper two PAC-Bayesian analyses.
 The first one (Section~\ref{sec:landmark}) is obtained by combining $n$ PAC-Bayesian bounds: instead of considering all the possible pairs of data points, we fix one point and we study the generalization ability for all the pairs involving it.
The second analysis (Section~\ref{sec:revisited}) is based on the fact that the loss can be expressed as a second-order U-statistics.

\section{PAC-BAYESIAN ANALYSIS AND LANDMARKS}
\label{sec:landmark}

Due to the linearity of the loss function $\loss$, we can rewrite the loss of $k_q$ as the \mbox{$q$-average} loss of every hypothesis. Indeed, Equation~\eqref{eq:genloss1} becomes
\begin{align*}
\Lcal_\Dcal(k_q)\ &=\ 
\Esp_{(\deltabf,\lambda)\sim \Delta_\Dcal} \ell\Big(\Esp_{\omegabf\sim q} h_\omegabf(\deltabf), \lambda\Big) \\
& = \
\Esp_{\omegabf\sim q} \ \Esp_{(\deltabf,\lambda)\sim \Delta_\Dcal }
\ell(h_\omegabf(\deltabf), \lambda) 
\, 
=  \Esp_{\omegabf\sim q} \ \Lcal_\Dcal (h_\omegabf)
\,.
\end{align*}
The above \mbox{$q$-expectation} of losses $\Lcal_\Dcal (h_\omegabf)$ turns out to be the quantity bounded by most PAC-Bayesian generalization theorems 
(sometimes referred as the \emph{Gibbs risk} in the literature),
excepted that such results usually apply to the loss over samples instead of distances.
Hence, we use PAC-Bayesian bounds to obtain generalization guarantees on $\Lcal_\Dcal(k_q)$ from its empirical estimate of Equation~\eqref{eq:emploss1}, that we can rewrite as
\begin{align*}
\widehat\Lcal_S(k_q)=  
\frac{1}{n^2{-}n} 
\hspace{-1mm} \sum_{i,j=1; i\ne j}^{n}
\hspace{-3mm} \ell\Big(\Esp_{\omegabf\sim q} h_\omegabf(\deltabf), \lambda_{ij}\Big)\!
= \hspace{-1mm} \Esp_{\omegabf\sim q} 
\hspace{-1mm} \widehat\Lcal_{S} (h_\omegabf)
\,.
\end{align*}
However the \emph{classical} PAC-Bayesian theorems cannot be applied directly to bound $\Lcal_\Dcal(k_q)$, as the empirical loss $\widehat\Lcal_S(k_q)$ would require to be computed from \iid\ observations of $\Delta_\Dgen$.
Instead, the empirical loss involves \emph{dependent} samples, as it is computed from \mbox{$n^2{-}n$} pairs formed by $n$ elements from~$\Dgen$. 

\subsection{First Order $\KL$-Bound}

 A straightforward approach to  apply \emph{classical} PAC-Bayesian results is to bound separately the loss associated with each training sample. That is, for each $(\xbf_i, y_i) \in S$, we define 
\begin{align} \label{eq:Lcali}
	\Lcal^i_{\Dcal}(h_\omegabf) &\eqdef 
	\Esp_{(\xbf,y)\sim \Dcal} \ell\Big( h_\omegabf(\xbf_i-\xbf), \lambda(y_i,y)\Big)\,,\\
	\nonumber
	\mbox{and }\widehat\Lcal^i_{S}(h_\omegabf) &\eqdef 
	\frac{1}{n{-}1}\sum_{j=1, j\neq i}^n\!\!\ell\Big( h_\omegabf(\xbf_i-\xbf_j), \lambda(y_i,y_j)\Big)\,.
\end{align}
Thus, the next theorem gives a generalization guarantee on $\Lcal^i_\Dcal(k_q)$ relying namely on the empirical estimate $\widehat\Lcal^i_{S}(k_q)$ and the Kullback-Leibler divergence 
\mbox{$\KL(q\|p) = \Esp_{\omegabf\sim q} \ln \frac{q(\omegabf)}{p(\omegabf)}$} between the prior~$p$ and the learned posterior~$q$. Note that the statement of Theorem~\ref{thm:pb1} is obtained straightforwardly from~\citet[][\mbox{Theorem~4.1} and \mbox{Lemma~1}]{alquier-16}, but can be recovered easily from~\cite{lever-13}.

\begin{thm} 
	\label{thm:pb1}
	For $t>0$,   $i\in\{1,\ldots,n\}$,  and a prior distribution $p$ over $\Rbb^d$, with probability $1{-}\varepsilon$ over the choice of~$S\sim \Dgen^{n}$, we have for all $q$ on $\Rbb^d$\,:
	\begin{equation*}
	\Lcal^i_{\Dgen}(k_q)  \leq \widehat\Lcal^i_{S}(k_q) + \frac{1}{t}\left( \KL(q\|p) + \frac{t^2}{2(n{-}1)} + \ln\frac1\varepsilon\right).
	\end{equation*}
\end{thm}

By the union bound, and using the fact that 
$\Lcal_\Dcal(k_q) = \Esp_{(\xbf_i, y_i)\sim \Dcal} \Lcal^i_{\Dcal}(k_q) $, we prove the following corollary in the supplementary material. 
\begin{cor} 
	\label{cor:pb1}
	For $t>0$ and a prior distribution $p$ over $\Rbb^d$, with probability $1{-}\varepsilon$ over the choice of~$S\sim \Dgen^{n}$, we have for all $q$ on $\Rbb^d$\,:
	\begin{equation*}
	\Lcal_{\Dgen}(k_q)  \leq \widehat\Lcal_{S}(k_q) + \frac{2}{t}\!\left( \KL(q\|p) + \frac{t^2}{2(n{-}1)} + \ln\frac{n{+}1}\varepsilon\right).
	\end{equation*}
\end{cor}

\paragraph{Pseudo-Posterior for $\KL$-bounds.}
Since the above result
is valid for any distribution~$q$, one can compute the bound for any learned posterior distribution.
Note that the bound promotes the minimization of a trade-off---parameterized by a constant $t$---between the empirical loss $\widehat\Lcal_{S}(k_q)$ and the $\KL$-divergence between the prior $p$ and the posterior~$q$:
$$
\widehat\Lcal_{S}(k_q)  + \frac2{t}\,\KL(q\|p)\,.$$
It is well-known that for fixed $t$, $p$ and $S$, the minimum bound value is obtained with the \emph{pseudo-Bayesian} posterior $q^*$, such that for $\omegabf\in \Rbb^d$,
\begin{equation}\label{eq:post_catoni}
q^*(\omegabf) \, = \, \frac1{Z}\,p(\omegabf) \,\exp{\left(-\tau\, \widehat\Lcal_{S}(h_\omegabf )	\right) }\,,
\end{equation}
where $\tau\eqdots\frac12 t$ and $Z$ is a normalization constant.\footnote{This trade-off is the same one involved in some other PAC-Bayesian bounds for \iid~data \citep[\eg,][]{catoni-07}.
	As discussed in \cite{zhang-06,grunwald-2012,germain-2016}, there is a similarity between the minimization of such PAC-Bayes bounds and the Bayes update rule.}
Note also Corollary~\ref{cor:pb1}'s bound converges to the generalization loss $\Lcal_{\Dgen}(k_q)$ at rate $O\Big(\sqrt{\frac{\ln n}{n}}\Big)$ for the parameter choice $t=\sqrt{n\ln n}$.

Due to the continuity of the feature space, the pseudo-posterior of Equation~\eqref{eq:post_catoni} is hard to compute. To estimate it, one may make use of Monte Carlo~\citep[\eg,][]{dalalyan12} or variational Bayes methods~\citep[\eg,][]{alquier-16}.
In this work, we explore a simpler method:
we work solely from a discrete probability space.

\subsection{Landmarks-Based Learning}
\label{sec:landmarks-based}
We now propose to leverage on the fact that Theorem~\ref{thm:pb1} bounds the kernel function for the distances to a single data point, instead of learning a kernel globally for every data point as in Corollary~\ref{cor:pb1}.
We thus aim at learning a collection of kernels (which we can also interpret as similarity functions) for a subset of the training points.
We call \emph{landmarks} these training points.
The aim of this approach is to learn a new representation of the input space, mapping the data-points into compact feature vectors, from  which we can learn a simple predictor.

Concretely, along with the learning sample $S$ of $n$ examples \iid\ from $\Dgen$, we consider a landmarks sample $L=\{(\xbf_l,y_l)\}_{l=1}^{n_L}$ of $n_L$  points \iid\ from $\Dgen$, and a \emph{prior} Fourier transform distribution $p$. For each landmark $(\xbf_l,y_l)\in L$, let sample $D$ points from~$p$, denoted $\Omega^L = \{\omegabf^l_m\}_{m=1}^D \sim p^D$. Then, consider a uniform distribution $P$ on the discrete hypothesis set $\Omega^L$, such that $P(\omegabf^l_m) = \frac1D$ and $h^l_m(\deltabf) \eqdef \cos(\omegabf^l_m\cdot\deltabf)$. 
We aim at learning a set of kernels 
$\{\kQl\}_{l=1}^{n_L}$, where each $\kQl$
is obtained from a distinct $\xbf_l\in L$ with a fixed parameter $\beta > 0$, by computing the pseudo-posterior distribution~$Q^l$ given by
\begin{equation} \label{eq:discrete_pseudoQ}
Q^l_m \, = \, \frac1{Z_l}\,\exp\Big(-\beta\sqrt{n} \,
\widehat\Lcal^l_{S}(h^l_m)
\Big)\,,
\end{equation}
for $ m{=}1,\ldots, D$\,; $Z_l$ being the normalization constant. Note that Equation~\eqref{eq:discrete_pseudoQ} gives the minimum of Theorem~\ref{thm:pb1} with
$t=\beta\sqrt{n}$. That is, $\beta=1$ corresponds to the regime where the bound converges.
Moreover, similarly to Corollary~\ref{cor:pb1}, generalization guarantees are obtained simultaneously for the $n_L$ computed distributions thanks to the union bound and Theorem~\ref{thm:pb1}. Thus, with probability \mbox{$1{-}\varepsilon$}, for all $\{Q^l\}_{l=1}^{n_L}$:
\begin{equation*}
\Lcal^l_\Dcal(\kQl)  \leq  \widehat\Lcal^l_{S}(\kQl) {+} \frac{1}{t}\!\left( \!\KL(Q^l\|P) {+} \frac{t^2}{2(n{-}1)} {+}\ln\!\frac{n_L}\varepsilon\right)\!,
\end{equation*}
where $ \KL(Q^l\|P) = \ln D + \sum_{j=1}^D Q^l_j \ln Q^l_j$.

Once all pseudo-posterior are computed thanks to Equation~\eqref{eq:discrete_pseudoQ}, our landmarks-based approach is to map samples $\xbf\in\Rbb^d$ to $n_L$ similarity features:
\begin{equation}\label{eq:psimap}
\psibf(\xbf)\eqdef \Big( \widehat{k}_{Q^1}(\xbf_1 {-} \xbf), \ldots, \widehat{k}_{Q^{n_L}}(\xbf_{n_L} {-} \xbf)   \Big),
\end{equation}
and to learn a linear predictor on the transformed training set.
Note that, this mapping is not a kernel map anymore and is somehow similar to the mapping proposed by~\citet{BalcanBS08ML,BalcanBS08COLT,zantedeschi2018multiview}
for a similarity function that is more general than a kernel but fixed for each landmark.

\section{LEARNING KERNEL (REVISITED)}
\label{sec:revisited}

In this section, we present PAC-Bayesian theorems that directly bound the \emph{kernel alignment} generalization loss $\Lcal_\Dcal(k_q)$ on a ``pairwise'' probability distribution $\Delta_\Dcal$---as defined by Equation~\eqref{eq:genloss1}---even if the empirical loss $\widehat\Lcal_\Dcal(k_q)$ is computed on dependent samples. These bounds suggest a \emph{kernel alignment} (or \emph{kernel learning}) strategy similar to the one of \citet{SinhaD16}. 
We stress that our guarantees hold solely for the kernel alignment loss, but not for the predictor trained with this kernel. Hence, our proposed algorithm learns a kernel independently of the prediction method to be used downstream. This is in contrast with the \emph{one-step} frameworks of \cite{yang2015carte} and \cite{oliva2016bayesian}, which learn a mixture of random kernel features in a \emph{fully} Bayesian way; they rely on a data-generating model, whereas our approach assumes only that the observations are \iid

\subsection{Second Order $\KL$-bound} 

The following result 
is based on the fact that 
$\widehat\Lcal_{S}(h_\omegabf)\eqdef\frac{1}{n^2-n} \sum_{i\ne j}^{n}
\ell(h_\omegabf(\deltabf_{ij}), \lambda_{ij})$ is an \emph{unbiased} \mbox{second-order} estimator of 
$\Esp_{(\deltabf,\lambda)\sim \Delta_\Dcal }
\ell(h_\omegabf(\deltabf), \lambda)$, allowing us to build on the PAC-Bayesian analysis for \mbox{U-statistics} of~\citet[][Theorem~7]{lever-13}.
Indeed, the next theorem gives a generalization guarantee on the kernel alignment loss $\Lcal_\Dcal(k_q)$.

\begin{thm}[\citealt{lever-13}]  \label{thm:ustats}
	For $t>0$  and a prior distribution $p$ over $\Rbb^d$, with probability $1{-}\varepsilon$ over the choice of~$S\sim \Dgen^{n}$, we have for all $q$ on~$\Rbb^d$\,:
	\begin{equation*}
	\Lcal_{\Dgen}(k_q)  \leq \widehat\Lcal_{S}(k_q) + \frac{1}{t}\left( \KL(q\|p) + \frac{t^2}{2n} + \ln\frac1\epsilon\right).
	\end{equation*}
\end{thm}
Except for some constant terms, the above Theorem~\ref{thm:ustats} is similar to Corollary~\ref{cor:pb1}. Indeed, both are minimized by the same pseudo-posterior $q^*$ (Equation~\ref{eq:post_catoni}, with $\tau\eqdots t$ for Theorem~\ref{thm:ustats}). Interestingly, we get rid of the $\ln (n+1)$ term of Corollary~\ref{cor:pb1}, making in Theorem~\ref{thm:ustats}'s bound to converge at rate $O(\frac{1}{\sqrt{n}})$ when $t=\sqrt{n}$.

\subsection{Second Order Bounds for \mbox{$f$-Divergences}} 
In the following, we build on a recent result of~\citet{alquier-2018} to express a new family of PAC-Bayesian bounds for our dependent samples, where the $\KL$ term is replaced by other \mbox{$f$-divergences}. 

Given a convex function $f$ such that $f(1){=}0$, a \mbox{$f$-divergence} is given by
$D_f(q\|p) \eqdots \Esp_{\omegabf\sim p} f\big( \frac{q(\omegabf)}{p(\omegabf)}\big).$
The following theorem applies to \mbox{$f$-divergences} such that $f(x)=x^\etamu -1$.
\begin{thm} \label{thm:ustats2}
	For $\etamu>1$  and a prior distribution $p$ over $\Rbb^d$, with probability $1{-}\varepsilon$ over the choice of~$S\sim \Dgen^{n}$, we have for all $q$ on $\Rbb^d$\,:
	\begin{align*}
	&\Lcal_\Dgen(k_q)  \leq \widehat\Lcal_{S}(k_q)  \\
	&+\begin{cases}
	\left(\frac1{2\sqrt{n}} \right)^{\!\etamu-1}
	\!\!\Big(\!D_\etamu(q\|p)+1\Big)^\frac1\etamu \!\left(\frac{1}{ \varepsilon}\right)^{\!1{-}\frac1\etamu}\!\! & \mbox{if $1<\etamu\leq 2$,}\\[2mm]
	\left(\frac1{4{n}} \right)^{\!1-\frac1\etamu}
	\!\!\Big(\!D_\etamu(q\|p)+1\Big)^\frac1\etamu \!\left(\frac{1}{ \varepsilon}\right)^{\!1{-}\frac1\etamu}\!\! & \mbox{if $\etamu> 2$,}
	\end{cases}
	\end{align*}
	where \,
	$\displaystyle    D_\etamu(q\|p) \eqdots \Esp_{\omegabf\sim p} \left( \frac{q(\omegabf)}{p(\omegabf)}\right)^\etamu \! - 1\,.$
\end{thm}
\begin{proof} Let 
	$\displaystyle    \Mcal_\etamu \eqdots \Esp_{\omegabf\sim p} \Esp_{S'\sim D^n} \left|  \Lcal_{\Dgen}(h_\omegabf) - \widehat\Lcal_{S'}(h_\omegabf) \right|^\etamu$.
	
	We start from \citet[Theorem~1]{alquier-2018}:
	\begin{equation} \label{eq:alquier}
	\Lcal_{\Dgen}(k_q)  \leq \widehat\Lcal_{S}(k_q) + 
	\left(\frac{\Mcal_\etamu}{\varepsilon}\right)^{1-\frac1\etamu} \Big(D_\etamu(q\|p)+1\Big)^\frac1\etamu.
	\end{equation}
	Let us show
	$\Mcal_\etamu\leq\left(\frac1{2\sqrt{n}} \right)^{\etamu} $ for $1<\etamu\leq 2$\,:
	\begin{align}
	\Mcal_\etamu \nonumber
	&= \Esp_{\omegabf\sim p} \Esp_{S'\sim D^n}\left[ \left( \Lcal_{\Dgen}(h_\omegabf) - \widehat\Lcal_{S'}(h_\omegabf) \right)^2 \right]^{\frac\etamu2} \\ 
	&\leq \Esp_{\omegabf\sim p}\left[ \Esp_{S'\sim D^n} \left( \Lcal_{\Dgen}(h_\omegabf) - \widehat\Lcal_{S'}(h_\omegabf) \right)^2 \right]^{\frac\etamu2} \label{eq:jensen007}
	\\ \nonumber
	&= \Esp_{\omegabf\sim p}\left[ \Var_{S'\sim D^n} \left( \Lcal_{S'}(h_\omegabf)  \right) \right]^{\frac\etamu2} 
	\\
	&\leq \Esp_{\omegabf\sim p}\left[ \frac1{4n} \right]^{\frac\etamu2} 
	= \left[ \frac1{4n} \right]^{\frac\etamu2}. \label{eq:boucheron007}
	\end{align}
	Line~\eqref{eq:jensen007} is obtained by Jensen's inequality (since \mbox{$0<\frac\etamu2\leq1$}), and the inequality of Line~\eqref{eq:boucheron007} is proven by Lemma~\ref{lem:1sur4n} of the supplementary material. Note that the latter is based on the Efron-Stein inequality and \citet[Corollary 3.2]{boucheron-13}.
	
	The first case of Theorem \ref{thm:ustats2} statement ($1<\etamu\leq 2$) is obtained by inserting Line~\eqref{eq:boucheron007}  in Equation~\eqref{eq:alquier}. The second case ($\etamu>2$) is obtained by upper-bounding $\Mcal_\etamu$ by $\Mcal_2 = \frac{1}{4n}$,
	as $ |\Lcal_{\Dgen}(h_\omegabf) - \widehat\Lcal_{S'}(h_\omegabf)|\leq 1$\,.
\end{proof}
As a particular case, with $\etamu=2$, we obtain from Theorem~\ref{thm:ustats2} a bound that relies on the chi-square divergence 	
	$\chi^2(q\|p) = \Esp_{\omegabf\sim p} \big( \frac{q(\omegabf)}{p(\omegabf)}\big)^2\! - 1\,$. 
\begin{cor} \label{cor:ustats+chi}
	Given a prior distribution $p$ over $\Rbb^d$, with probability $1{-}\varepsilon$ over the choice of~$S\sim \Dgen^{n}$, we have for all $q$ on $\Rbb^d$\,:
	\begin{equation*}
	\Lcal_{\Dgen}(k_q)  \leq \widehat\Lcal_{S}(k_q) + 
	\sqrt{\frac{\chi^2(q\|p)+1}{4\, n\, \varepsilon}}\,.
	\end{equation*}
\end{cor}
It is noteworthy that the above result looks alike other PAC-Bayesian bounds based on the \mbox{chi-square} divergence in the \iid setting, as the one of \citet[Lemma~7]{honorio-14}, \citet[Corollary~10]{graal-aistats16} or \citet[Corollary~1]{alquier-2018}.
Interestingly, the latter has been introduced to handle unbounded (possibly heavy-tailed) losses, and one could also extend our Corollary~\ref{cor:ustats+chi} to this setting.

\subsection{PAC-Bayesian Interpretation of Kernel Alignment Optimization}
\label{section:alignement}

\citet{SinhaD16} propose a kernel learning algorithm that weights random kernel features. To do so, their algorithm solves a \emph{kernel alignment} problem. As explained below, this method is coherent with the PAC-Bayesian theory exposed by our current work. 

\paragraph{Kernel alignment algorithm.}
Let us consider a Fourier transform distribution $p$, from which $N$ points are sampled, denoted $\Omega \!=\! \{\omegabf_m\}_{m=1}^N \sim p^N$. Then, consider a uniform distribution $P$ on the discrete hypothesis set $\Omega$, such that $P(\omegabf_m) = \frac1N$ and $h_m(\deltabf) \eqdef \cos(\omegabf_m\cdot\deltabf)$. 
Given a dataset $S$, and constant parameters $\etamu>1$, $\rho>0$, the optimization algorithm proposed by~\citeauthor{SinhaD16} solves the following problem.
\begin{align} \label{eq:sinhaduchi1}
&\maximize_{Q\in \Rbb^N_+} \sum_{i=1}^n \sum_{j=1}^n  
\lambda_{ij} \sum_{m=1}^N Q_m h_m(\deltabf_{ij})\,,\\
\label{eq:sinhaduchi2}
&\quad\mbox{such that } \sum_{m=1}^N Q_m {=} 1 \mbox{ and } D_\etamu(Q\|P) \leq \rho\,.
\end{align}
The iterative procedure proposed by~\citeauthor{SinhaD16} finds an \mbox{$\epsilon$-suboptimal} solution to the above problem in $O(N \log(\frac1\epsilon))$ steps.
The solution provides a learned kernel \mbox{$\widehat{k}_Q(\deltabf)\eqdef \frac1N\sum_{m=1}^N Q_m h_m(\deltabf)$}.

\citeauthor{SinhaD16} propose to use the above alignment method to reduce the number of features needed compared to the classical RFF procedure (as described in Section~\ref{sec:rffsetting}). Albeit this method is a kernel learning one, empirical experiments show that with a large number of random features, the classical RFF procedure achieves as good prediction accuracy. However, one can draw (with replacement) $D<N$ features from $\Omega$ according to $Q$. For a relatively small $D$, learning a linear predictor on the random feature vector (such as the one presented by Equation~\ref{eq:phibf_mapping}) obtained from $Q$ achieves better results than the classical RFF method on the same number $D$ of random features.

\paragraph{PAC-Bayesian interpretation.} The optimization problem of Equations~(\ref{eq:sinhaduchi1}--\ref{eq:sinhaduchi2}) deals with the same trade-off as the one promoted by Theorem~\ref{thm:ustats2}. Indeed, maximizing Equation~\eqref{eq:sinhaduchi1} amounts to minimizing $\widehat\Lcal_{S}(k_q)$, and the constraint of Equation~\eqref{eq:sinhaduchi2} controls the $f$-divergence $D_\etamu(Q\|P)$, which is the same complexity measure involved in Theorem~\ref{thm:ustats2}. Furthermore, the empirical experiments performed by~\citet{SinhaD16} focus on the \mbox{$\chi^2$-divergence} (case $\etamu{=}2$), which corresponds to tackling the trade-off expressed by Corollary~\ref{cor:ustats+chi}.

\subsection{Greedy Kernel Learning}
\label{sect:greedy_kernel_learning}

The method proposed by \citet{SinhaD16} can easily be adapted to minimize the bound of Theorem~\ref{thm:ustats} instead of the bound of Theorem~\ref{thm:ustats2}. 
We describe this kernel learning procedure below.

Given a Fourier transform prior distribution $p$, let sample $N$ points $\Omega = \{\omegabf_m\}_{m=1}^N \sim p^N$. Let $P(\omegabf_m) = \frac1N$ and $h_m(\deltabf) \eqdef \cos(\omegabf_m\cdot\deltabf)$. 
Given a dataset $S$, and constant parameters $\beta>0$, compute the following pseudo-posterior for $m=1,\ldots, N\,$:
\begin{equation} \label{eq:discrete_pseudoQ_final}
Q_m\, = \, \frac1{Z}\,\exp\Big(-\beta\sqrt{n} \,
\widehat\Lcal_{S}(h_m)
\Big)\,.
\end{equation}
%
Then, we sample with replacement $D<N$ features from $\Omega$ according to the pseudo-posterior $Q$. The sampled features are used to map every $\xbf\in\Rbb^d$ of the training set into a new vector $\phibf(\xbf)\in\Rbb^{2D}$ according to Equation~\eqref{eq:phibf_mapping}. The latter transformed dataset is then given as input to a linear learning procedure.

In summary, this learning method is strongly inspired by the one described in Section~\ref{section:alignement}, but the posterior computation phase is faster, as we benefit from a closed-form expression (Equation~\ref{eq:discrete_pseudoQ_final}).
Once $\widehat\Lcal_{S}(h_m)$ is computed for all $h_m$,\footnote{We show in supplementary material (section~\ref{section:kalc}) that each $\widehat\Lcal_{S}(h_m)$ can be computed in $O(n)$ steps.}
we can vary the parameter $\beta$ and get a new posterior in $O(N)$ steps.


\section{EXPERIMENTS}\label{sec:experiments}

All experiments use a Gaussian (\aka\ RBF) kernel of variance~$\sigma^2$: 
$k_\sigma (\xbf,\xbf')=\exp\big(-\tfrac1{2\sigma^2}\|\xbf-\xbf'\|^2\big)\,,$
for which the Fourier transform is given by
\begin{equation}\label{eq:rbf_prior}
p_\sigma (\omegabf) \ = \ \left(\tfrac{\sigma^2}{2 \pi}\right)^{\frac d 2} e^{-\frac12 \sigma^2 \|\omegabf\|^2} \ = \ \Ncal(\omegabf ; \zerobf, \tfrac1{\sigma^2} \Ibf)\,.
\end{equation}
Apart from the toy experiment of Figure~\ref{fig:pgtoy}, the experiments on real data are conducted by splitting the available data into a training set, a validation set and a test set. The kernel parameter $\sigma$ is chosen among $\{10^{-7}, 10^{-6}, \dots, 10^{2}\}$ by running an RBF SVM on the training set and keeping the parameter having the best accuracy score on the validation set. That is, this $\sigma$ defines the prior distribution given by Equation~\eqref{eq:rbf_prior} for all our pseudo-Bayesian methods. Unless otherwise specified, all the other parameters are selected using the validation set. More details about the experimental procedure are given in the supplementary material.

\subsection{Landmarks-Based Learning}
\label{section:expe_landmarks}


\begin{figure*}[h]\centering
	\makebox[\textwidth]{\includegraphics[width=.8\textwidth]{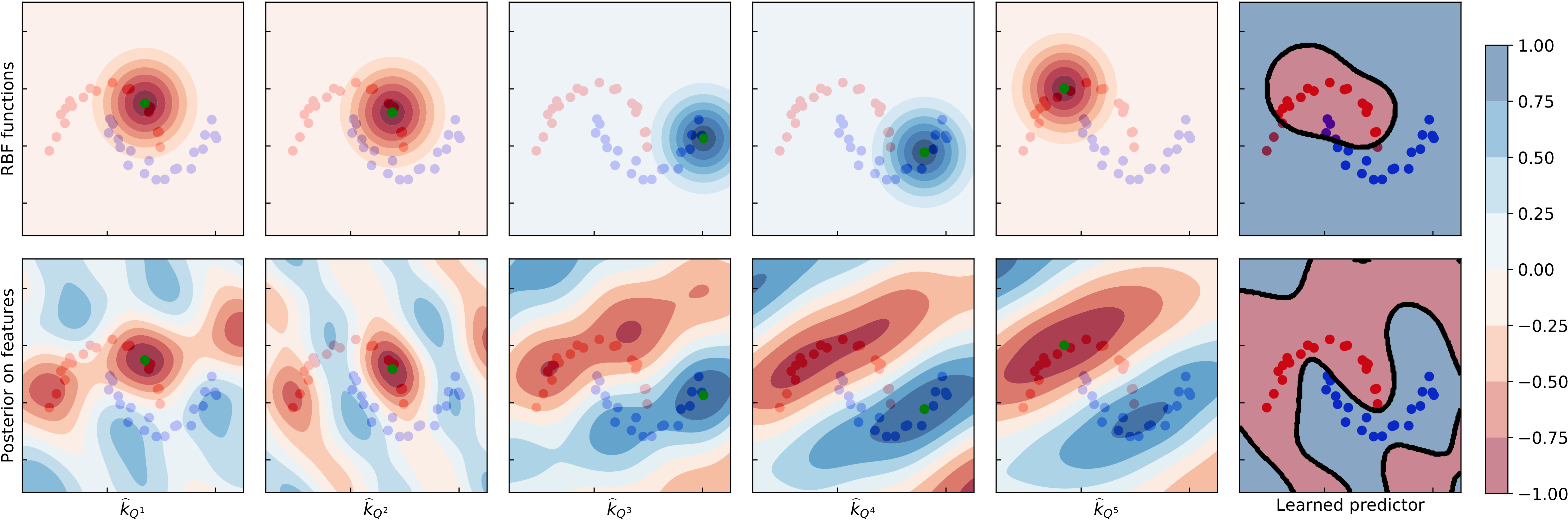}} \\   
	\caption{First row shows selected RBF-Landmarks kernel outputs, while second row shows the corresponding learned similarity measures on random Fourier features (PB-Landmarks). The rightmost column displays the classification learned by a linear SVM over the mapped dataset.}  \label{fig:pgtoy} 
\end{figure*}

\begin{figure}[h]
    \centering
    \includegraphics[trim={0cm 0cm 0cm 0cm},clip,width=\linewidth]{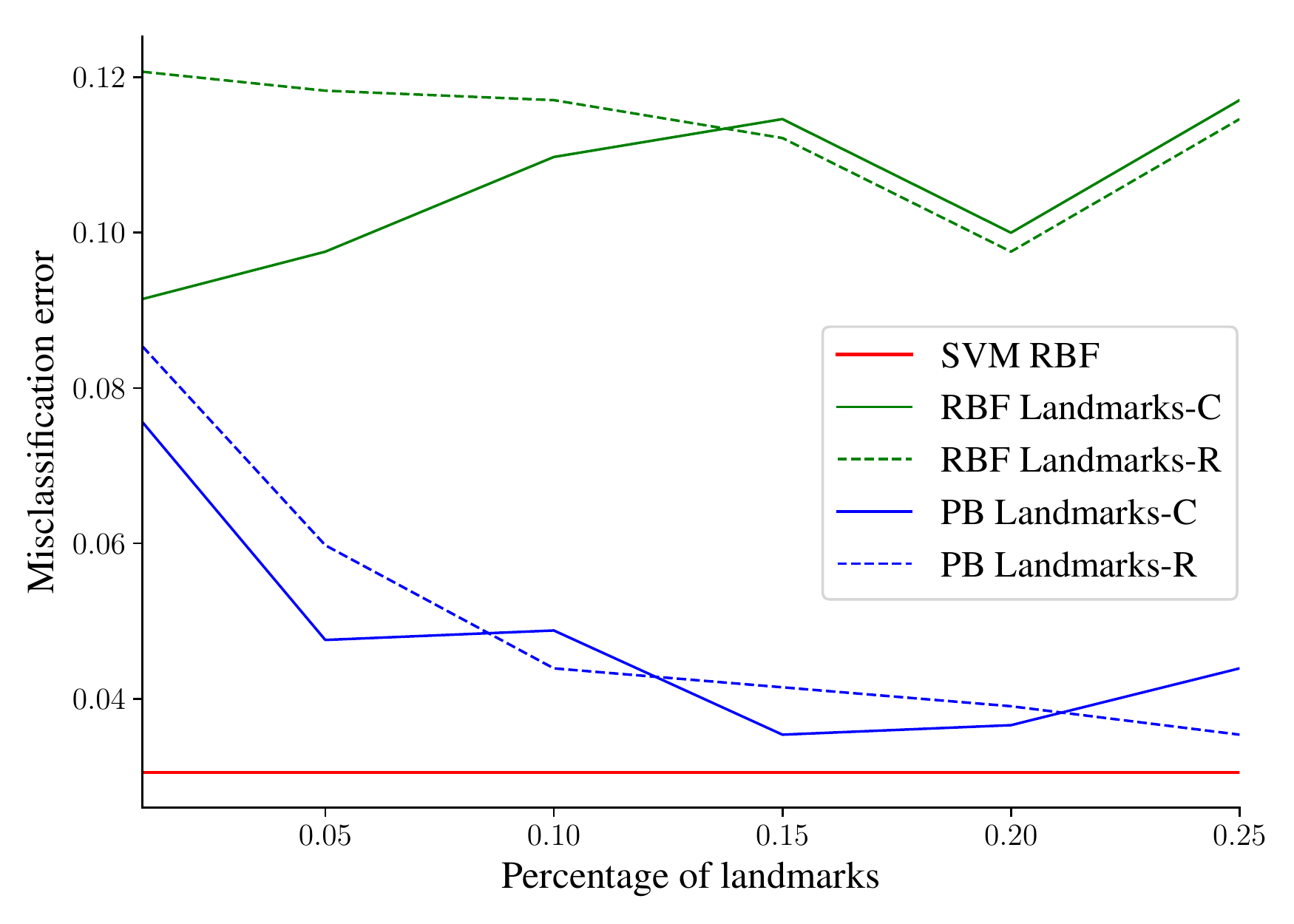}
    \caption{Behavior of the landmarks-based approach according to the percentage of training points selected as landmarks on the dataset ``ads''. }
    \label{fig:error_landmarks_ads}
\end{figure}

\paragraph{Toy experiment.} To get some insight from the landmarks-based procedure
of Section~\ref{sec:landmarks-based},
we generate a 2D dataset~$S_{\rm toy}$, illustrated by Figure~\ref{fig:pgtoy}.  
We randomly select five training points $L{=}\{\xbf_1,\xbf_2,\xbf_3,\xbf_4,\xbf_5\}\subset S_{\rm toy}$, and compare two procedures, described below.

\underline{RBF-Landmarks}: Learn a linear SVM on the \emph{empirical kernel map} given by the five RBF kernels centered on~$L$. That is, each $\xbf\in S_{\rm toy}$ is mapped such that\\[1mm]
	{\footnotesize $ 
	\xbf \mapsto \Big(k_\sigma(\xbf_1,\xbf), k_\sigma(\xbf_2,\xbf), k_\sigma(\xbf_3,\xbf), k_\sigma(\xbf_4,\xbf), k_\sigma(\xbf_5,\xbf)\Big)\,.
	$}
	
\underline{PB-Landmarks}: Generate $20$ random samples according to the Fourier transform of Equation~\eqref{eq:rbf_prior}. For every landmark of $L$, learn a \emph{similarity measure} thanks to Equation~\eqref{eq:discrete_pseudoQ} (with $\beta=1$), 
minimizing the PAC-Bayesian bound.
We thus obtain five posterior distributions $Q^1,Q^2,Q^3,Q^4,Q^5$, and learn a linear SVM on the mapped training set obtained by Equation~\eqref{eq:psimap}.

\begin{table}[h]
  \caption{Test error of the landmarks-based approach.}
  \label{tab:landmarks_results}
  \centering
  \setlength{\tabcolsep}{4pt}
  \begin{tabular}{lrrrrr}
    \toprule
    \multirow{2}{*}[-3pt]{Dataset} &    \multirow{2}{*}[-11pt]{} &  \multicolumn{4}{c}{landmarks-based}\\
    \cmidrule(lr){3-6}
    & SVM & RBF &    PB &  PB$_{\beta=1}$ &  PB$_{D=64}$ \\
    \midrule
    ads     &   3.05 &  10.98 &   \textbf{4.88} &     5.12 &   5.00 \\
    adult   &  19.70 &  19.60 &  \textbf{17.99} &    \textbf{17.99} &  \textbf{17.99} \\
    breast  &   4.90 &   6.99 &   3.50 &     3.50 &   \textbf{2.80} \\
    farm    &  11.58 &  17.47 &  15.73 &    \textbf{14.19} &  15.73 \\
    mnist17 &   0.34 &   0.74 &   0.42 &     \textbf{0.32} &   \textbf{0.32} \\
    mnist49 &   1.16 &   2.26 &   \textbf{1.80} &     2.09 &   2.50 \\
    mnist56 &   0.55 &   \textbf{0.97} &   1.06 &     1.55 &   1.03 \\
    \bottomrule
\end{tabular}
\end{table}

Hence, the RBF-Landmarks method corresponds to the prior, from which we learn a posterior by landmarks by the PB-Landmarks procedure. Right-most plots of Figure~\ref{fig:pgtoy} show that the PB-Landmarks setting successfully finds a representation from which the linear SVM can predict well.

\begin{figure*}[t]
    \centering
    \begin{subfigure}[t]{0.325\textwidth}
        \centering
        \includegraphics[trim={0cm 0cm 0cm 0cm},clip,width=\linewidth]{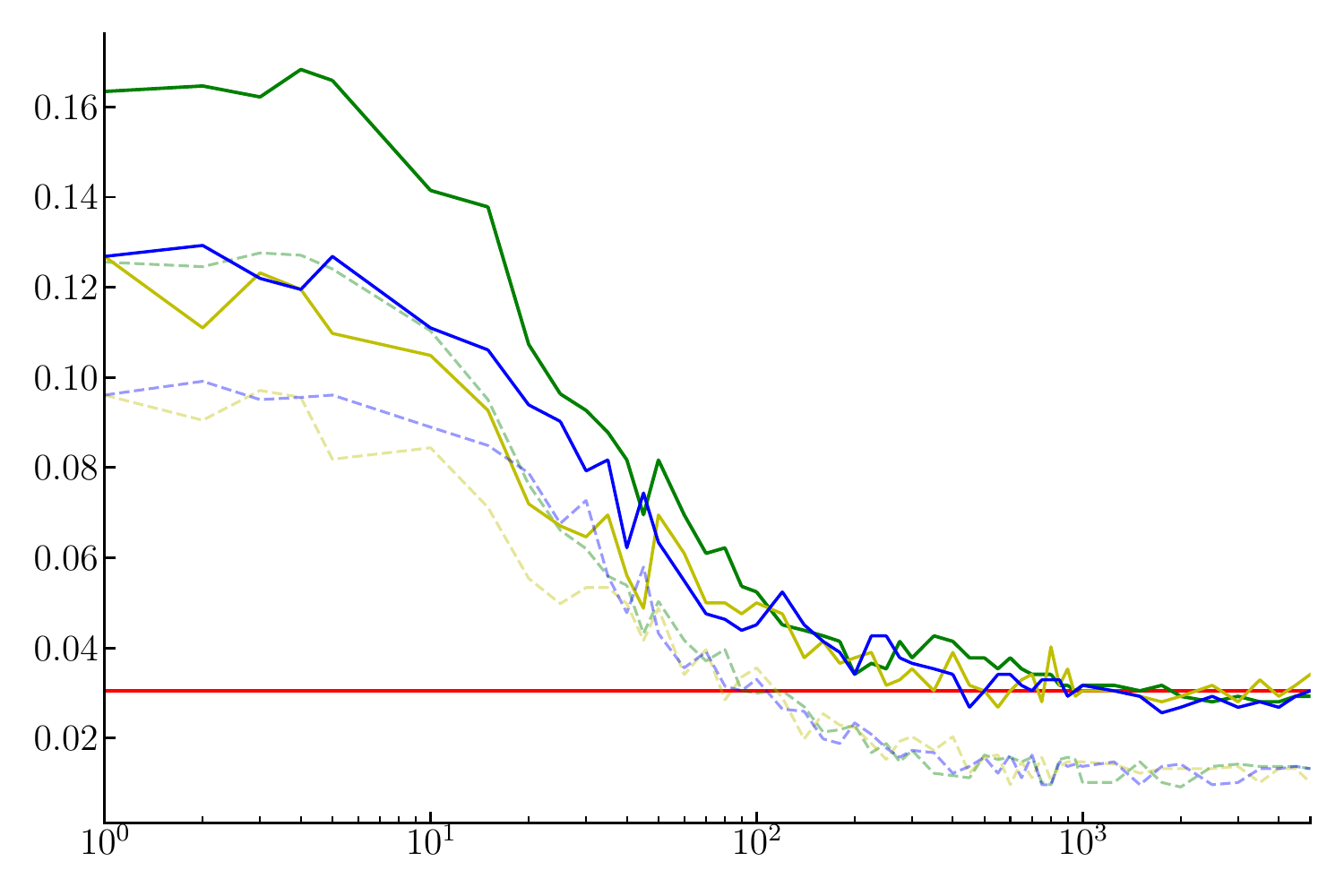}\vspace{-2mm}
        \caption{ads}
        \label{fig:err_vs_d_ads}
    \end{subfigure}
    \hfill
    \begin{subfigure}[t]{0.325\textwidth}
        \centering
        \includegraphics[trim={0cm 0cm 0cm 0cm},clip,width=\linewidth]{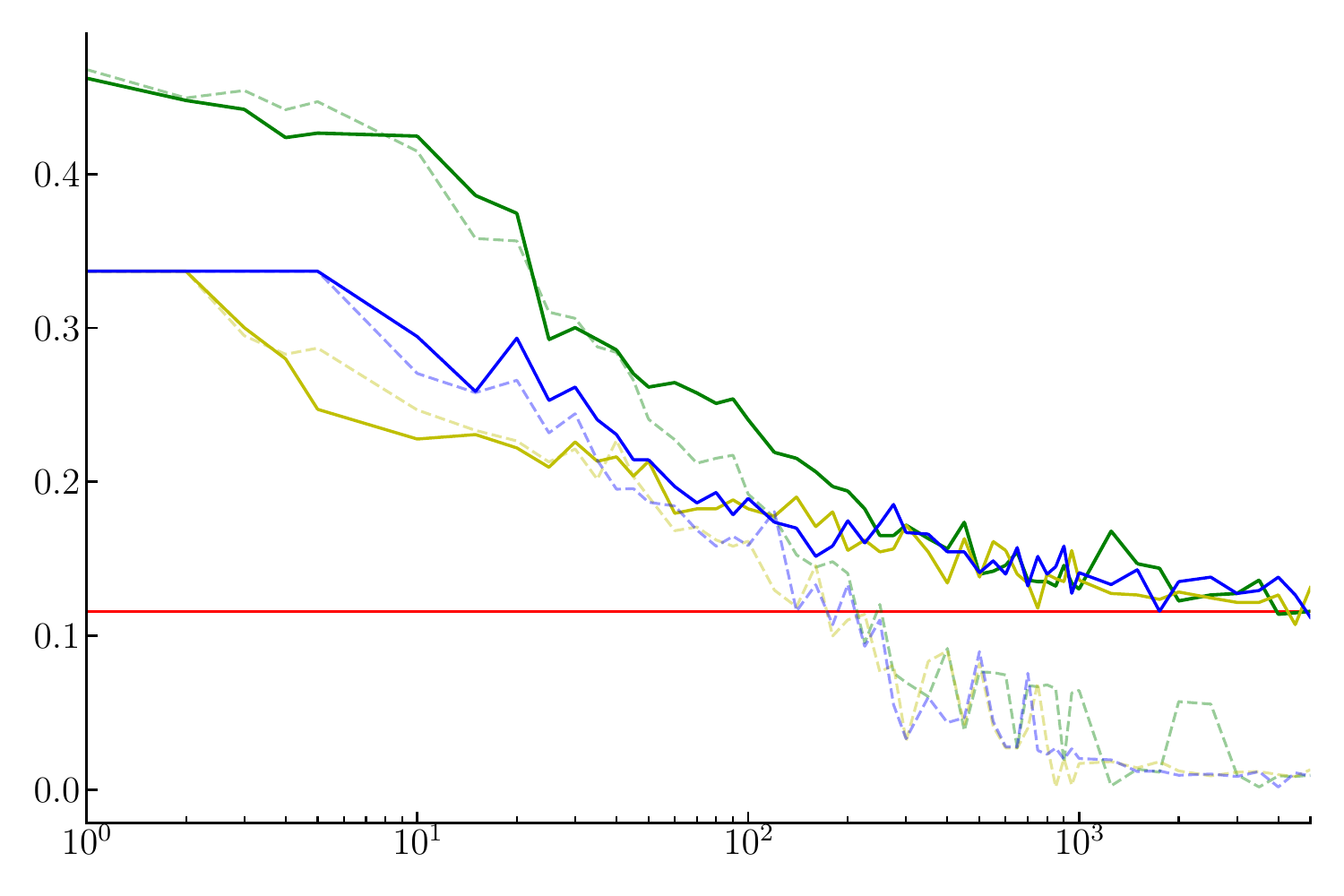}\vspace{-2mm}
        \caption{farm}
         \label{fig:err_vs_d_farm}
    \end{subfigure}
    \hfill
    \begin{subfigure}[t]{0.325\textwidth}
        \centering
        \includegraphics[trim={0cm 0cm 0cm 0cm},clip,width=\linewidth]{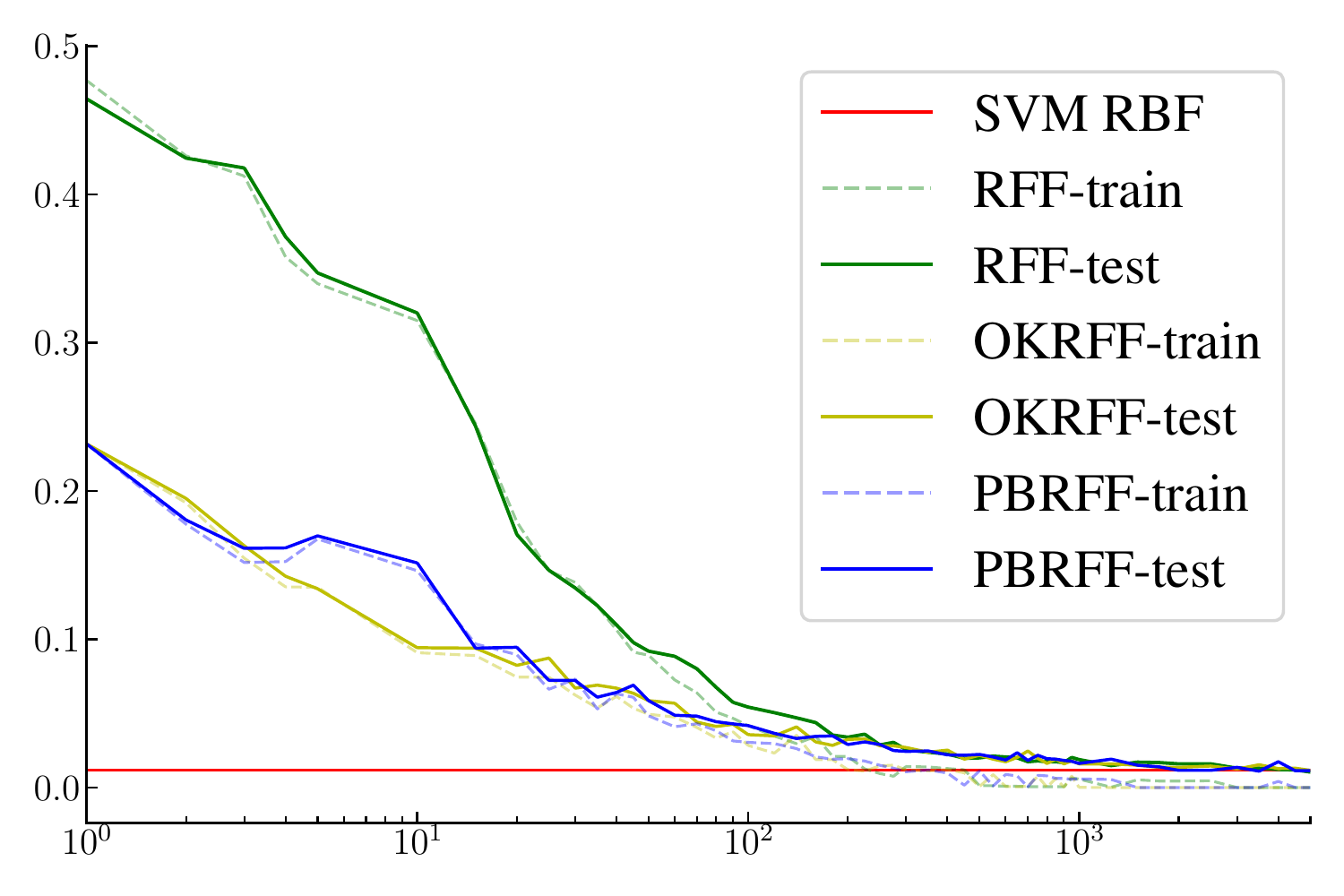}\vspace{-2mm}
        \caption{mnist49}
         \label{fig:err_vs_d_mnist49}
    \end{subfigure}
    \caption{Train and test error of the kernel learning approaches according to the number of random features $D$.}
    \label{fig:err_vs_d}
\end{figure*}

\paragraph{Experiments on real data.}

We conduct similar experiments as the above one on seven real binary classification datasets.
Figure~\ref{fig:error_landmarks_ads} studies the behavior of the approaches according to the number of selected landmarks. 
We select a percentage of the training points as landmarks (from $1\%$ to $25\%$), and we compare the classification error of a linear SVM on the mapping obtained by the original RBF functions (as in the RBF-Landmarks method above), with the mapping obtained by learning the landmarks posterior distributions (PB-Landmark method). 
We also compare the case where the landmarks are selected at random among the training data (curves postfixed~\mbox{``-R''}), to another scenario where we use the centroids obtained with a  $k$-Means clustering as landmarks (curves postfixed ``-C''). 
Note that, the latter case is not rigorously backed by our PAC-Bayesian theorems, since the choice of landmarks is now dependent of the whole observed training set. 
The results show that the classification error of both cases are similar, but the clustering strategy leads to a more stable behavior, probably since the landmarks are more representative of the original space. Moreover, the pseudo-Bayesian method improves the results on almost all datasets.

Table~\ref{tab:landmarks_results} compares the error rate of an SVM (trained along with the full Gram matrix and a properly selected $\sigma$ on the validation set) with four landmarks-based approaches: 
{\it (RBF)} the landmarks are RBF kernel of parameter~$\sigma$; 
{\it (PB)} the PB-Landmarks approach where the number of features per landmarks $D$ and the $\beta$ parameter are selected using the validation set; 
{\it (PB$_{\beta=1}$)} the PB-Landmarks approach where $\beta{=}1$ is fixed and $D$ is selected by validation; and
{\it (PB$_{D=64}$)} the PB-Landmarks approach where $D{=}64$ is fixed and $\beta$ is selected by validation. 
For all landmarks-based approaches, we select the landmarks by clustering, and use $10\%$ of the training set size as the number of landmarks; we want to study the methods in the regime where it provides relatively compact representations. 
We observe 
that learning the posterior improves the RBF-Landmarks (except on ``mnist56'') and that the validation of both $\beta$ and $D$ parameters are not mandatory to obtain satisfactory results. The SVM RBF is better than all landmarks-based approaches on 4 datasets out of 7,
but requires a far less compact representation of the data as it uses the full Gram matrix.

\subsection{Greedy Kernel Learning}
\label{section:expe_greedy}

Figure~\ref{fig:err_vs_d} presents a study of the kernel learning method detailed in Section~\ref{sect:greedy_kernel_learning}, inspired from the one of \citet{SinhaD16}.
We first generate $N{=}20000$ random features according to $p_\sigma$ as given by Equation~\eqref{eq:zcossin}, and we learn a posterior using two strategies: {\it (OKRFF)} the original optimized kernel of~\citeauthor{SinhaD16} given by Equations~(\ref{eq:sinhaduchi1}-\ref{eq:sinhaduchi2}), where $\rho$ is selected 
on the validation set; and {\it (PBRFF)} the pseudo-posterior given by Equation~\eqref{eq:discrete_pseudoQ_final} where $\beta$ is selected on the validation set. For both obtained posteriors, we subsample an increasing number of features $D\in[1,5000]$ to create the mapping given by Equation~\eqref{eq:phibf_mapping}, on which we learn a linear SVM. We also compare to {\it (RFF)} the standard random Fourier features as described in Section~\ref{sec:rffsetting}, with $D$ randomly selected features according to the prior $p_\sigma$.

We see that our PBRFF approach behaves similarly as OKRFF, with a slight advantage for the latter. However, we recall that computing the posterior of former method is faster. 
Both kernel learning methods have better accuracy than the classical RFF algorithm for a small number of random features, and similar ones for a large number of random features.

\section{\!CONCLUSION \& PERSPECTIVES}
\label{sec:conclu}
We elaborated an original viewpoint of the random Fourier features, proposed by~\cite{rahimi-07} to approximate a kernel.
By looking at the Fourier transform as a prior distribution over trigonometric functions, we present two kinds of generalization theorems that bound a kernel alignment loss. Based on classical first-order PAC-Bayesian results, we derived a landmarks-based  strategy that learns a compact representation of the data. Then, we proposed two second-order generalization bounds. The first one is based on the \mbox{U-statistic} theorem of~\citet{lever-13}. The second one is a new PAC-Bayesian theorem for \mbox{$f$-divergences} (replacing the usual $\KL$-divergence term).  We show that the latter bound provides a theoretical justification to the kernel alignment method of~\citet{SinhaD16}, and we also empirically evaluate a similar but simpler algorithm where the alignment distribution is obtained by the PAC-Bayesian pseudo-posterior closed-form expression. 

Our current guarantees hold solely for the kernel alignment loss, and not for the predictor trained with this kernel. An important research direction is to extend the guarantees to the final predictor, which could in turn be the bedrock of a new one-step learning procedure \citep[in the vein of][]{yang2015carte,oliva2016bayesian}.
Other research directions include the study of the RKHS associated with the learned kernel,
and the extension of our study to wavelet transforms \citep{Mallat-book}.
Furthermore, considering the Fourier transform of a kernel as a \mbox{(pseudo-)}Bayesian prior might lead to other original contributions.
Among them, we foresee new perspectives on representation and metric learning, namely for unsupervised learning.

\paragraph{Acknowledgments.}
P. Germain wants to thank Francis Bach for insightful preliminary discussions.
This work was supported in part by the French Project APRIORI ANR-18-CE23-0015 and in part by NSERC. 
This research was enabled in part by support provided by  Compute Canada {\footnotesize (\url{www.computecanada.ca})}.

\bibliographystyle{plainnat}
\bibliography{biblio}

\begin{thebibliography}{31}
\providecommand{\natexlab}[1]{#1}
\providecommand{\url}[1]{\texttt{#1}}
\expandafter\ifx\csname urlstyle\endcsname\relax
  \providecommand{\doi}[1]{doi: #1}\else
  \providecommand{\doi}{doi: \begingroup \urlstyle{rm}\Url}\fi

\bibitem[Alquier and Guedj(2018)]{alquier-2018}
Pierre Alquier and Benjamin Guedj.
\newblock Simpler {PAC-B}ayesian bounds for hostile data.
\newblock \emph{Machine Learning}, 107\penalty0 (5), 2018.

\bibitem[Alquier et~al.(2016)Alquier, Ridgway, and Chopin]{alquier-16}
Pierre Alquier, James Ridgway, and Nicolas Chopin.
\newblock On the properties of variational approximations of gibbs posteriors.
\newblock \emph{Journal of Machine Learning Research}, 17, 2016.

\bibitem[Bach(2017)]{bach-17-equivalence}
Francis~R. Bach.
\newblock On the equivalence between kernel quadrature rules and random feature
  expansions.
\newblock \emph{Journal of Machine Learning Research}, 18, 2017.

\bibitem[Balcan et~al.(2008{\natexlab{a}})Balcan, Blum, and
  Srebro]{BalcanBS08COLT}
Maria{-}Florina Balcan, Avrim Blum, and Nathan Srebro.
\newblock Improved guarantees for learning via similarity functions.
\newblock In \emph{{COLT}}, 2008{\natexlab{a}}.

\bibitem[Balcan et~al.(2008{\natexlab{b}})Balcan, Blum, and
  Srebro]{BalcanBS08ML}
Maria{-}Florina Balcan, Avrim Blum, and Nathan Srebro.
\newblock A theory of learning with similarity functions.
\newblock \emph{Machine Learning}, 72\penalty0 (1-2):\penalty0 89--112,
  2008{\natexlab{b}}.

\bibitem[B\'egin et~al.(2016)B\'egin, Germain, Laviolette, and
  Roy]{graal-aistats16}
Luc B\'egin, Pascal Germain, Fran\c{c}ois Laviolette, and Jean-Francis Roy.
\newblock {PAC-Bayesian} bounds based on the {R}\'enyi divergence.
\newblock In \emph{AISTATS}, 2016.

\bibitem[Boser et~al.(1992)Boser, Guyon, and Vapnik]{boser-92}
Bernhard~E. Boser, Isabelle Guyon, and Vladimir Vapnik.
\newblock A training algorithm for optimal margin classifiers.
\newblock In \emph{COLT}, 1992.

\bibitem[Boucheron et~al.(2013)Boucheron, Lugosi, and Massart]{boucheron-13}
Stéphane Boucheron, Gábor Lugosi, and Pascal Massart.
\newblock \emph{Concentration inequalities : a nonasymptotic theory of
  independence}.
\newblock Oxford university press, 2013.
\newblock ISBN 978-0-19-953525-5.

\bibitem[Catoni(2007)]{catoni-07}
Olivier Catoni.
\newblock \emph{PAC-Bayesian supervised classification: the thermodynamics of
  statistical learning}, volume~56.
\newblock Inst. of Mathematical Statistic, 2007.

\bibitem[Choromanski et~al.(2018)Choromanski, Rowland, Sarl{\'{o}}s, Sindhwani,
  Turner, and Weller]{Choromanski18}
Krzysztof Choromanski, Mark Rowland, Tam{\'{a}}s Sarl{\'{o}}s, Vikas Sindhwani,
  Richard~E. Turner, and Adrian Weller.
\newblock The geometry of random features.
\newblock In \emph{{AISTATS}}, 2018.

\bibitem[Dalalyan and Tsybakov(2012)]{dalalyan12}
Arnak~S. Dalalyan and Alexandre~B. Tsybakov.
\newblock Sparse regression learning by aggregation and langevin monte-carlo.
\newblock \emph{J. Comput. Syst. Sci.}, 78\penalty0 (5), 2012.

\bibitem[Drineas and Mahoney(2005)]{drineas2005nystrom}
Petros Drineas and Michael~W Mahoney.
\newblock On the nystr{\"o}m method for approximating a gram matrix for
  improved kernel-based learning.
\newblock \emph{Journal of Machine Learning Research}, 6\penalty0 (Dec), 2005.

\bibitem[Germain et~al.(2016)Germain, Bach, Lacoste, and
  Lacoste{-}Julien]{germain-2016}
Pascal Germain, Francis~R. Bach, Alexandre Lacoste, and Simon Lacoste{-}Julien.
\newblock {PAC-B}ayesian theory meets {B}ayesian inference.
\newblock In \emph{{NIPS}}, 2016.

\bibitem[Goodfellow et~al.(2016)Goodfellow, Bengio, and
  Courville]{Goodfellow-16-book}
Ian Goodfellow, Yoshua Bengio, and Aaron Courville.
\newblock \emph{Deep Learning}.
\newblock MIT Press, 2016.
\newblock \url{http://www.deeplearningbook.org}.

\bibitem[Gr{\"{u}}nwald(2012)]{grunwald-2012}
Peter Gr{\"{u}}nwald.
\newblock The safe {B}ayesian - learning the learning rate via the mixability
  gap.
\newblock In \emph{{ALT}}, 2012.

\bibitem[Honorio and Jaakkola(2014)]{honorio-14}
Jean Honorio and Tommi~S. Jaakkola.
\newblock Tight bounds for the expected risk of linear classifiers and
  pac-bayes finite-sample guarantees.
\newblock In \emph{AISTATS}, 2014.

\bibitem[Lever et~al.(2013)Lever, Laviolette, and Shawe{-}Taylor]{lever-13}
Guy Lever, Fran{\c{c}}ois Laviolette, and John Shawe{-}Taylor.
\newblock Tighter {PAC-B}ayes bounds through distribution-dependent priors.
\newblock \emph{Theor. Comput. Sci.}, 473, 2013.

\bibitem[Mallat(2008)]{Mallat-book}
St{\'{e}}phane Mallat.
\newblock \emph{A Wavelet Tour of Signal Processing, 3rd Edition}.
\newblock Academic Press, 2008.

\bibitem[McAllester(1999)]{mcallester-99}
David McAllester.
\newblock Some {PAC}-{B}ayesian theorems.
\newblock \emph{Machine Learning}, 37\penalty0 (3), 1999.

\bibitem[Oliva et~al.(2016)Oliva, Dubey, Wilson, P{\'o}czos, Schneider, and
  Xing]{oliva2016bayesian}
Junier~B Oliva, Avinava Dubey, Andrew~G Wilson, Barnab{\'a}s P{\'o}czos, Jeff
  Schneider, and Eric~P Xing.
\newblock Bayesian nonparametric kernel-learning.
\newblock In \emph{{AISTATS}}, 2016.

\bibitem[Rahimi and Recht(2007)]{rahimi-07}
Ali Rahimi and Benjamin Recht.
\newblock Random features for large-scale kernel machines.
\newblock In \emph{{NIPS}}, 2007.

\bibitem[Rudi and Rosasco(2017)]{Rudi17}
Alessandro Rudi and Lorenzo Rosasco.
\newblock Generalization properties of learning with random features.
\newblock In \emph{{NIPS}}, 2017.

\bibitem[Shawe{-}Taylor and Cristianini(2004)]{shawe-taylor-04-book}
John Shawe{-}Taylor and Nello Cristianini.
\newblock \emph{Kernel Methods for Pattern Analysis}.
\newblock Cambridge University Press, 2004.

\bibitem[Sinha and Duchi(2016)]{SinhaD16}
Aman Sinha and John~C. Duchi.
\newblock Learning kernels with random features.
\newblock In \emph{{NIPS}}, 2016.

\bibitem[Vapnik(1998)]{vapnik-98}
Vladimir Vapnik.
\newblock \emph{Statistical learning theory}.
\newblock Wiley, 1998.

\bibitem[Williams and Seeger(2001)]{williams2001nystrom}
Christopher K.~I. Williams and Matthias Seeger.
\newblock Using the {N}ystr\"{o}m method to speed up kernel machines.
\newblock In \emph{NIPS}. 2001.

\bibitem[Yang et~al.(2012)Yang, Li, Mahdavi, Jin, and Zhou]{nystromVSrff}
Tianbao Yang, Yu-feng Li, Mehrdad Mahdavi, Rong Jin, and Zhi-Hua Zhou.
\newblock Nystr\"{o}m method vs random fourier features: A theoretical and
  empirical comparison.
\newblock In \emph{NIPS}. 2012.

\bibitem[Yang et~al.(2015)Yang, Wilson, Smola, and Song]{yang2015carte}
Zichao Yang, Andrew~Gordon Wilson, Alexander~J. Smola, and Le~Song.
\newblock A la carte - learning fast kernels.
\newblock In \emph{{AISTATS}}, 2015.

\bibitem[Yu et~al.(2016)Yu, Suresh, Choromanski, Holtmann{-}Rice, and
  Kumar]{Yu16}
Felix~X. Yu, Ananda~Theertha Suresh, Krzysztof~Marcin Choromanski, Daniel~N.
  Holtmann{-}Rice, and Sanjiv Kumar.
\newblock Orthogonal random features.
\newblock In \emph{{NIPS}}, 2016.

\bibitem[Zantedeschi et~al.(2018)Zantedeschi, Emonet, and
  Sebban]{zantedeschi2018multiview}
Valentina Zantedeschi, R{\'e}mi Emonet, and Marc Sebban.
\newblock Fast and provably effective multi-view classification with
  landmark-based svm.
\newblock In \emph{ECML-PKDD}, 2018.

\bibitem[Zhang(2006)]{zhang-06}
Tong Zhang.
\newblock Information-theoretic upper and lower bounds for statistical
  estimation.
\newblock \emph{{IEEE} Trans. Information Theory}, 52\penalty0 (4), 2006.

\end{thebibliography}

\newcounter{lemsupIndex}
\addtocounter{lemsupIndex}{\thethm}

\newtheorem{innercustomthm}{Lemma}
\newenvironment{lemsup}
  { \addtocounter{lemsupIndex}{1}
  \renewcommand\theinnercustomthm{S\thelemsupIndex}\innercustomthm}
  {\endinnercustomthm}

\newtheorem{innercustomthmB}{Theorem}
\newenvironment{thmsup}
  { \addtocounter{lemsupIndex}{1}
  \renewcommand\theinnercustomthmB{S\thelemsupIndex}\innercustomthmB}
  {\endinnercustomthmB}

\onecolumn
\thispagestyle{empty}

\appendix
\section{Supplementary Material}

\subsection{Mathematical Results}


\paragraph{Corollary~\ref{cor:pb1}.}
\textit{For $t>0$ and a prior distribution $p$ over $\Rbb^d$, with probability $1{-}\varepsilon$ over the choice of~$S\sim \Dgen^{n}$, we have for all $q$ on $\Rbb^d$\,:
	\begin{equation*}
	\Lcal_{\Dgen}(k_q)  \leq \widehat\Lcal_{S}(k_q) + \frac{2}{t}\left( \KL(q\|p) + \frac{t^2}{2(n-1)} + \ln\frac{n+1}\varepsilon\right).
	\end{equation*}
}
\begin{proof}
We want to bound
\begin{align*}
\Lcal_\Dcal(k_q) 
&= \Esp_{(\xbf, y)\sim \Dcal}\Esp_{(\xbf', y')\sim \Dcal}\Esp_{\omegabf \sim q} \ell\Big( h_\omegabf(\xbf-\xbf'), \lambda(y,y')\Big) \\
&= \Esp_{(\xbf', y')\sim \Dcal} \Lcal'_\Dcal (k_q)\,,
\end{align*}
where $\Lcal'_\Dcal (k_q)$ is the alignment loss of the kernel $k_q$ centered on $(\xbf', y') \sim \Dcal$ (see Equation~\eqref{eq:Lcali}).

Let $t>0$ and $p$ a distribution on $\Rbb^d$. By applying the PAC-Bayesian theorem, with $\varepsilon_0\in(0,1)$, we have \begin{align*}
\Pr_{S\sim\Dcal^n} \left( \forall q \mbox{ on } \Rbb^d :
\Lcal_\Dcal(k_q) 
\leq \frac{1}{n}\sum_{i=1}^{n}\Lcal^i_\Dcal(k_q) + \frac{1}{t}\Bigg[\KL(q\|p) + \frac{t^2}{2n} + \ln\frac{1}{\varepsilon_0}\Bigg] \right) \geq 1-\epsilon_0\,.
\end{align*}
Moreover, we have that for each $i \in \{1,\ldots,n\}$, with a $\varepsilon_i\in(0,1)$, we have
\begin{align*}
\Pr_{S\sim\Dcal^n} \left( \forall q \mbox{ on } \Rbb^d :
\Lcal^i_\Dcal(k_q)
\leq
\widehat\Lcal^i_S(k_q) + \frac{1}{t}\Bigg[\KL(q\|p) + \frac{t^2}{2(n-1)} + \ln\frac{1}{\varepsilon_i}\Bigg] \right) \geq 1-\epsilon_i\,.
\end{align*}
By combining above probabilistic results with $\varepsilon_0=\varepsilon_1=\cdots=\varepsilon_n = \frac{\varepsilon}{n+1}$, we obtain that, with probability at least $1-\varepsilon$\,,
\begin{align*}
\Lcal_\Dcal(k_q) &=  \Esp_{(\xbf', y')\sim \Dcal} \Lcal'_\Dcal (k_q)\\
&\leq 
\frac{1}{n}\sum_{i=1}^{n} \Bigg[ 
\widehat\Lcal^i_S(k_q) + \frac{1}{t}\bigg[\KL(q\|p) + \frac{t^2}{2(n-1)} + \ln\frac{n+1}{\varepsilon}\bigg] \Bigg]
+ \frac{1}{t}\bigg[\KL(q\|p) + \frac{t^2}{2n} + \ln\frac{n+1}{\varepsilon}\bigg] \\
&=
\widehat\Lcal_S(k_q) + \frac{1}{t}\bigg[\KL(q\|p) + \frac{t^2}{2(n-1)} + \ln\frac{n+1}{\varepsilon}\bigg] 
+ \frac{1}{t}\bigg[\KL(q\|p) + \frac{t^2}{2n} + \ln\frac{n+1}{\varepsilon}\bigg] \\
&\leq
\widehat\Lcal_S(k_q)
+ \frac{2}{t}\bigg[\KL(q\|p) + \frac{t^2}{2(n-1)} + \ln\frac{n+1}{\varepsilon}\bigg]\,.
\end{align*}
\end{proof}

\begin{lem} \label{lem:1sur4n} For any data-generating distribution $\Dcal$:
$$\Var_{S'\sim \Dcal^n} \left( \Lcal_{S'}(h_\omegabf)  \right) \leq \frac1{4n} \,.$$
\end{lem}

\begin{proof}
Given $S'=\{(\xbf_i,y_i)\}_{i=1}^n\sim\Dcal^n$, we denote  
$$\Fcal_\omegabf\big(S') \eqdef \Fcal_\omegabf\big((\xbf_1, y_1), \ldots, (\xbf_n,y_n)\big) \eqdots \Lcal_{S'}(h_\omegabf) = \frac{1}{n(n-1)}\sum_{i\neq j}^n   \ell\Big( h_\omegabf(\xbf_i-\xbf_j), \lambda(y_i, y_j)\Big)\,.$$
The function $\Fcal_\omegabf$ above has the \emph{bounded differences property}. That is, for each $i\in\{1,\ldots,n\}$\,:
$$\sup_{S', \xbf^*\in\Rbb^d, y^*\in Y} \left|
\Fcal_\omegabf\big((\xbf_1, y_1), \ldots, (\xbf_n,y_n)\big) 
-
\Fcal_\omegabf\big((\xbf_1, y_1), \ldots, (\xbf_{i-1}, y_{i-1}), (\xbf^*, y^*), (\xbf_{i+1}, y_{i+1}), \ldots, (\xbf_n,y_n)\big) 
\right| \leq \frac1n,$$
Thus, we apply the Efron-Stein inequality \citep[following][Corollary~3.2]{boucheron-13} to obtain 
$$\Var_{S'\sim \Dcal^n} \left( \Fcal_\omegabf (S') \right)\leq \frac14 \sum_{i=1}^n \left(\frac1n\right)^2 = \frac1{4n}\,. $$
\end{proof}

\subsection{Kernel Alignment Loss Computation}
\label{section:kalc}

The kernel learning algorithms presented in Section~\ref{sec:revisited} require to compute the empirical kernel alignment loss for each hypothesis $h_\omegabf$, given by
\begin{align} \label{eq:naif}
        \widehat\Lcal_{S}(h_\omegabf) &= \frac{1}{n(n{-}1)}\sum_{i\ne j}^{n} \ell(h_\omegabf(\deltabf_{ij}), \lambda_{ij})\,.
\end{align}
A naive implementation of Equation~\eqref{eq:naif} would need $O(n^2)$ steps. Propositions~\ref{prop:O(n)-binary} and~\ref{prop:O(n)-multiclass} below show how to rewrite Equation~\eqref{eq:naif} in a form that needs $O(n)$ steps. Proposition~\ref{prop:O(n)-binary} is dedicated to the binary classification, and is equivalent to the computation method proposed by \citet{SinhaD16}. By Proposition~\ref{prop:O(n)-multiclass}, we extend the result to the multi-classification case.

\begin{prop}[Binary classification]\label{prop:O(n)-binary} When $S={(\xbf_i, y_i)}_{i=1}^n \in (\Rbb^d \times \{-1, 1\})^n$, we have

\begin{equation*}
     \widehat\Lcal_{S}(h_\omegabf) =  \frac{n}{2(n{-}1)}  - \frac{1}{2n(n{-}1)}\left[\left(\sum_{i=1}^{n}y_i\cos(\omegabf \cdot \xbf_i)\right)^{2} + \left(\sum_{i=1}^{n}y_i\sin(\omegabf \cdot \xbf_i)\right)^{2} \right]\,.
\end{equation*}
\end{prop}
That  is, in the binary classification case ($y \in  \{-1,1\}$), one can compute the empirical alignment loss $\widehat\Lcal_{S}(h_\omegabf)$ in $O(n)$ steps.
\begin{proof} Using the cosine trigonometric identity
\begin{align*}
    \sum_{i\ne j}^{n} \lambda_{ij}h_\omegabf(\xbf_i-\xbf_j) &= \sum_{i=1}^{n}\sum_{j=1}^{n} y_i y_j \cos(\omegabf \cdot (\xbf_i{-}\xbf_j)) -n\\
    &= \sum_{i=1}^{n}\sum_{j=1}^{n} y_i y_j\left(\cos(\omegabf \cdot \xbf_i)\cos(\omegabf \cdot \xbf_j)+\sin(\omegabf \cdot \xbf_i)\sin(\omegabf \cdot \xbf_j)\right) - n\\
    &= \left(\sum_{i=1}^{n}y_i\cos(\omegabf \cdot \xbf_i)\right)^{2} + \left(\sum_{i=1}^{n}y_i\sin(\omegabf \cdot \xbf_i)\right)^{2} - n
\end{align*}
Thus, 
\begin{align*}
    \widehat\Lcal_{S}(h_\omegabf) &= \frac{1}{n(n{-}1)}\sum_{i\ne j}^{n} \ell(h_\omegabf(\deltabf_{ij}), \lambda_{ij})\\
    &= \frac{1}{n(n{-}1)}\sum_{i\ne j}^{n} \frac{1 - \lambda_{ij}h_\omegabf(\deltabf_{ij})}{2}\\
    &= \frac{1}{2} - \frac{1}{2n(n{-}1)}\sum_{i\ne j}^{n} \lambda_{ij}h_\omegabf(\xbf_i{-}\xbf_j)\\
    &= \frac{1}{2} - \frac{1}{2n(n{-}1)}\left[\left(\sum_{i=1}^{n}y_i\cos(\omegabf \cdot \xbf_i)\right)^{2} + \left(\sum_{i=1}^{n}y_i\sin(\omegabf \cdot \xbf_i)\right)^{2} - n\right]\\
    &=  \frac{n}{2(n{-}1)} - \frac{1}{2n(n{-}1)}\left[\left(\sum_{i=1}^{n}y_i\cos(\omegabf \cdot \xbf_i)\right)^{2} + \left(\sum_{i=1}^{n}y_i\sin(\omegabf \cdot \xbf_i)\right)^{2} \right]\,.
\end{align*}
\end{proof}

\begin{prop}[Multi-class classification]\label{prop:O(n)-multiclass} When $S={(\xbf_i, y_i)}_{i=1}^n \in (\Rbb^d \times \{1,\ldots,L\})^n$, we have \begin{equation*}
\widehat\Lcal_{S}(h_\omegabf) =  \frac{n}{2(n{-}1)} - \frac{1}{2n(n{-}1)}\left[2 \sum_{y=1}^L (c_y^2 + s_y^2) - \left(\sum_{y=1}^L c_y\right)^2 - \left(\sum_{y=1}^L s_y\right)^2 \right]\,,
\end{equation*}
with 
\begin{align*}
    c_y \eqdef  \sum_{\xbf\in S_y}\cos(\omegabf \cdot \xbf)\, \quad \mbox{ and } \quad
    s_y \eqdef  \sum_{\xbf\in S_y}\sin(\omegabf \cdot \xbf)\,.
\end{align*}
\end{prop} 
That  is, in the multi-class classification case with $L$ classes ($y \in  \{1,\ldots,L\})^n$), one can compute the empirical alignment loss $\widehat\Lcal_{S}(h_\omegabf)$ in $O(n)$ steps.
\begin{proof}
\begin{align*}
    \sum_{i\ne j}^{n} \lambda_{ij}h_\omegabf(\xbf_i-\xbf_j) &= \sum_{i=1}^{n}\sum_{j=1}^{n} \lambda_{ij} \cos(\omegabf \cdot (\xbf_i{-}\xbf_j)) -n\\
    &= \sum_{i=1}^{n}\sum_{j=1}^{n} (2 I[y_i = y_j]-1) \cos(\omegabf \cdot (\xbf_i{-}\xbf_j)) -n\\
    &= 2\sum_{i=1}^{n}\sum_{j=1}^{n} I[y_i = y_j] \cos(\omegabf \cdot (\xbf_i{-}\xbf_j)) - \sum_{i=1}^{n}\sum_{j=1}^{n}  \cos(\omegabf \cdot (\xbf_i{-}\xbf_j)) -n\\
\end{align*}

Let's denote $S_y \eqdef \{\xbf_i|(\xbf_i, y)\in S\}$. We have 
\begin{align*}
\sum_{i=1}^{n}\sum_{j=1}^{n} I[y_i = y_j] \cos(\omegabf \cdot (\xbf_i{-}\xbf_j)) 
&= \sum_{y=1}^L\sum_{\xbf\in S_y}\sum_{\xbf'\in S_y} \cos(\omegabf \cdot (\xbf{-}\xbf'))\\
&=  \sum_{y=1}^L \left[ \left(\sum_{\xbf\in S_y}\cos(\omegabf \cdot \xbf)\right)^{2} + \left(\sum_{\xbf\in S_y}\sin(\omegabf \cdot \xbf)\right)^{2} \right],
\end{align*}
and
\begin{align*}
\sum_{i=1}^{n}\sum_{j=1}^{n}  \cos(\omegabf \cdot (\xbf_i{-}\xbf_j))
&=  \left(\sum_{y=1}^L \sum_{\xbf\in S_y}\cos(\omegabf \cdot \xbf)\right)^{2} + \left(\sum_{y=1}^L \sum_{\xbf\in S_y}\sin(\omegabf \cdot \xbf)\right)^{2}.
\end{align*}
Thus, we can rewrite
\begin{align*}
     \sum_{i\ne j}^{n} \lambda_{ij}h_\omegabf(\xbf_i-\xbf_j) = 
     2 \sum_{y=1}^L (c_y^2 + s_y^2) - \left(\sum_{y=1}^L c_y\right)^2 - \left(\sum_{y=1}^L s_y\right)^2 - n\,.
\end{align*}
\newpage
Therefore,

\vspace{-10mm}
\begin{align*}
    \widehat\Lcal_{S}(h_\omegabf) &= \frac{1}{n(n{-}1)}\sum_{i\ne j}^{n} \ell(h_\omegabf(\deltabf_{ij}), \lambda_{ij})\\
    &= \frac{1}{n(n{-}1)}\sum_{i\ne j}^{n} \frac{1 - \lambda_{ij}h_\omegabf(\deltabf_{ij})}{2}\\
    &= \frac{1}{2} - \frac{1}{2n(n{-}1)}\sum_{i\ne j}^{n} \lambda_{ij}h_\omegabf(\xbf_i{-}\xbf_j)\\
    &= \frac{1}{2} - \frac{1}{2n(n{-}1)}\left[2 \sum_{y=1}^L (c_y^2 + s_y^2) - \left(\sum_{y=1}^L c_y\right)^2 - \left(\sum_{y=1}^L s_y\right)^2 - n\right]\\
    &=  \frac{n}{2(n{-}1)} - \frac{1}{2n(n{-}1)}\left[2 \sum_{y=1}^L (c_y^2 + s_y^2) - \left(\sum_{y=1}^L c_y\right)^2 - \left(\sum_{y=1}^L s_y\right)^2 \right]\,.\\&\qedhere
\end{align*}
\end{proof}

\subsection{Experiments}

\paragraph{Implementation details.} The code used to run the experiments is available at:
\begin{center}
    \url{https://github.com/gletarte/pbrff}
\end{center}
 In Section \ref{sec:experiments} we use the following datasets: 
\begin{description}
    \item[ads] http://archive.ics.uci.edu/ml/datasets/Internet+Advertisements\\
    The first 4 features which have missing values are removed.
    \item[adult] {https://archive.ics.uci.edu/ml/datasets/Adult}
    \item[breast] https://archive.ics.uci.edu/ml/datasets/Breast+Cancer+Wisconsin+(Diagnostic).
    \item[farm] https://archive.ics.uci.edu/ml/datasets/Farm+Ads
    \item[mnist] http://yann.lecun.com/exdb/mnist/ \\
     As \cite{SinhaD16}, binary classification tasks are compiled with the following digits pairs: 1 vs. 7, 4 vs. 9, and 5 vs. 6.
\end{description}

We split the datasets into training and testing sets with a 75/25 ratio except for adult which has a training/test split already computed. We then use 20\% of the training set for validation. Table \ref{tab:datasets_overview} presents an overview. We use the following parameter values range for selection on the validation set:

\begin{figure}[h!]
\begin{floatrow}
\ffigbox{%
\begin{itemize}   
    \item $C \in \{10^{-5}, 10^{-4}, \dots, 10^{4}\}$ 
    
    \item $\sigma \in \{10^{-7}, 10^{-6}, \dots, 10^{2}\}$
    
     \item $\rho \in \{10^{-4}N, 10^{-3}N, \dots, 10^{0}N\}$ 
     
     \item $\beta \in \{10^{-3}, 10^{-2}, \dots, 10^{3}\}$
     
     \item $D \in \{8, 16, 32, 64, 128\}$
     \item[]
\end{itemize}
}{}
\capbtabbox{%
  \centering
  \setlength{\tabcolsep}{4pt}
  \begin{tabular}{lrrrrr}
    \toprule
  
    Dataset & $n_{train}$ & $n_{valid}$ & $n_{test}$ &  $d$\\
    \midrule
    ads     &   1967 & 492 & 820 & 1554 \\
    adult   &  26048 & 6513 & 16281 & 108 \\
    breast  &   340 & 86 & 143 & 30 \\
    farm    &  2485 & 622 & 1036 & 54877\\
    mnist17 &   9101 & 2276 & 3793 & 784\\
    mnist49 &   8268 & 2068 & 3446 & 784 \\
    mnist56 &   7912 & 1979 & 3298 & 784 \\
    \bottomrule
\end{tabular}
}{%
  \caption{Datasets overview.}\label{tab:datasets_overview}%
}
\end{floatrow}
\end{figure}

\paragraph{Supplementary experiments.} Figures~\ref{fig:pgtoy_supp} and~\ref{fig:err_land_supp}
present extra results obtained for the \emph{landmarks-based learning} experiments (Subsection~\ref{section:expe_landmarks}). Figure~\ref{fig:err_vs_d_supp} gives extra results for the \emph{greedy kernel learning} experiment (Subsection~\ref{section:expe_greedy}).

    
    
     
     

  

\begin{figure*}[t] \centering
	\makebox[\textwidth]{\includegraphics[width=1\textwidth]{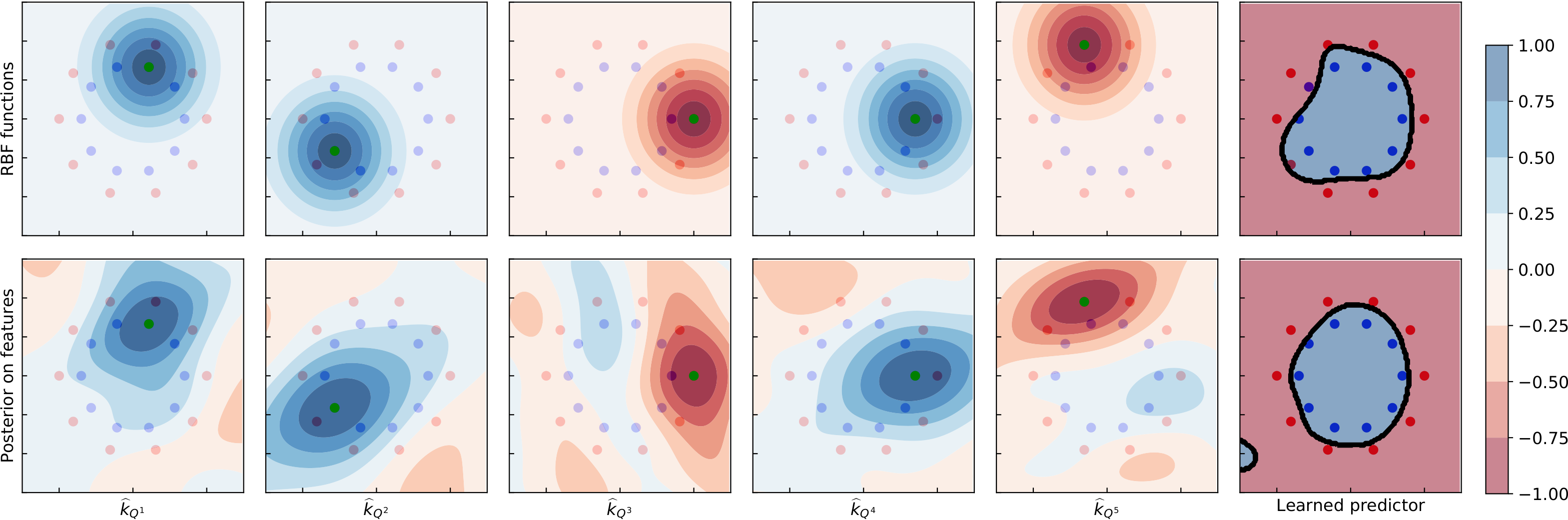}} \\  
	\caption{Repetition of Figure~\ref{fig:pgtoy}'s experiment, with another toy dataset.
	 First row shows selected RBF-Landmarks kernel outputs, while second row shows the corresponding learned similarity measures on random Fourier features (PB-Landmarks). The rightmost column displays the classification learned by a linear SVM over the mapped dataset.}  \label{fig:pgtoy_supp} 
\end{figure*}

\begin{figure*}[t!]
    \centering
    \begin{subfigure}[t]{0.48\textwidth}
        \centering
        \includegraphics[trim={0cm 0cm 0cm 0cm},clip,width=\linewidth]{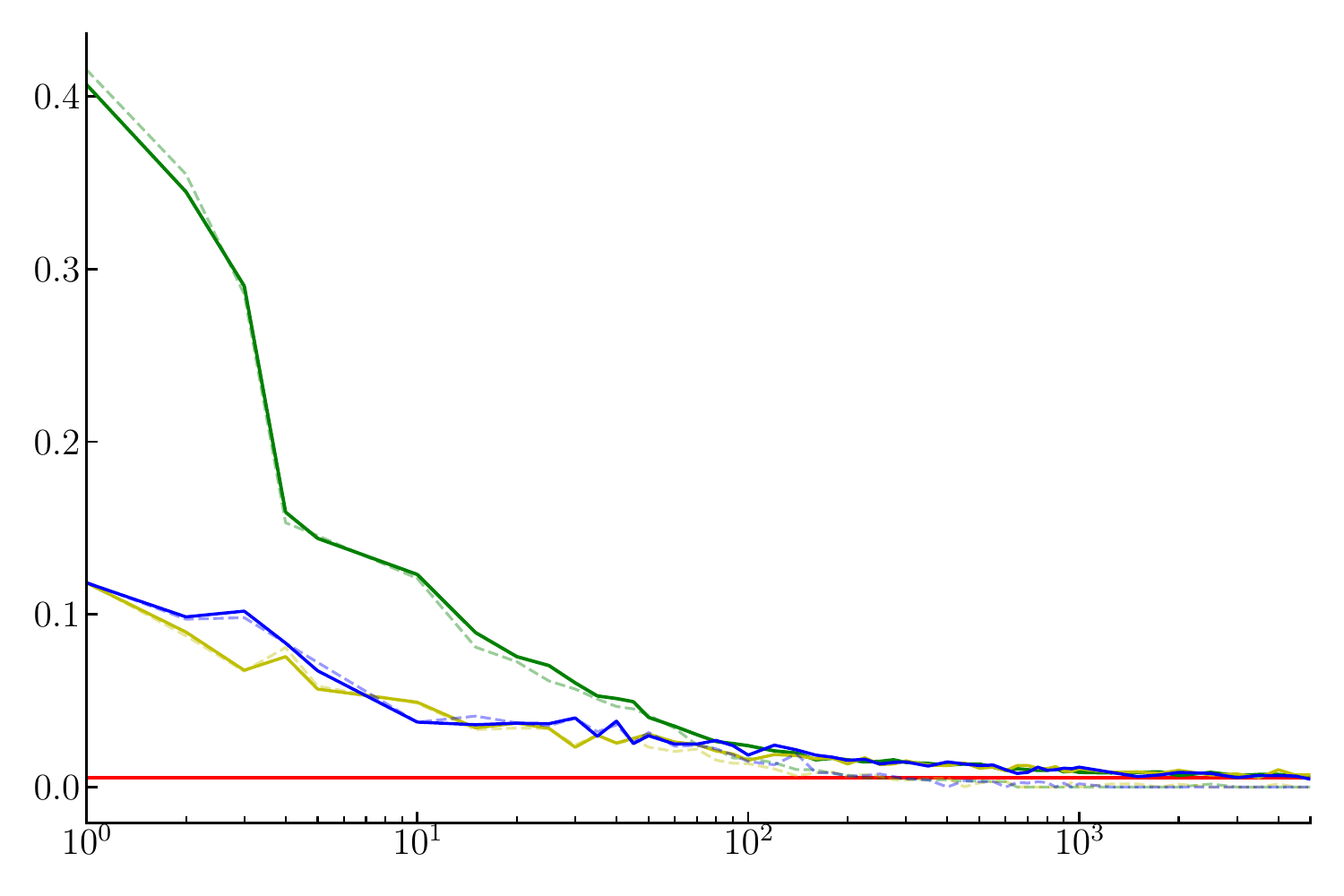}
        \caption{mnist56}
        \label{fig:err_vs_d_mnist56}
    \end{subfigure}
    \hfill
    \begin{subfigure}[t]{0.48\textwidth}
        \centering
        \includegraphics[trim={0cm 0cm 0cm 0cm},clip,width=\linewidth]{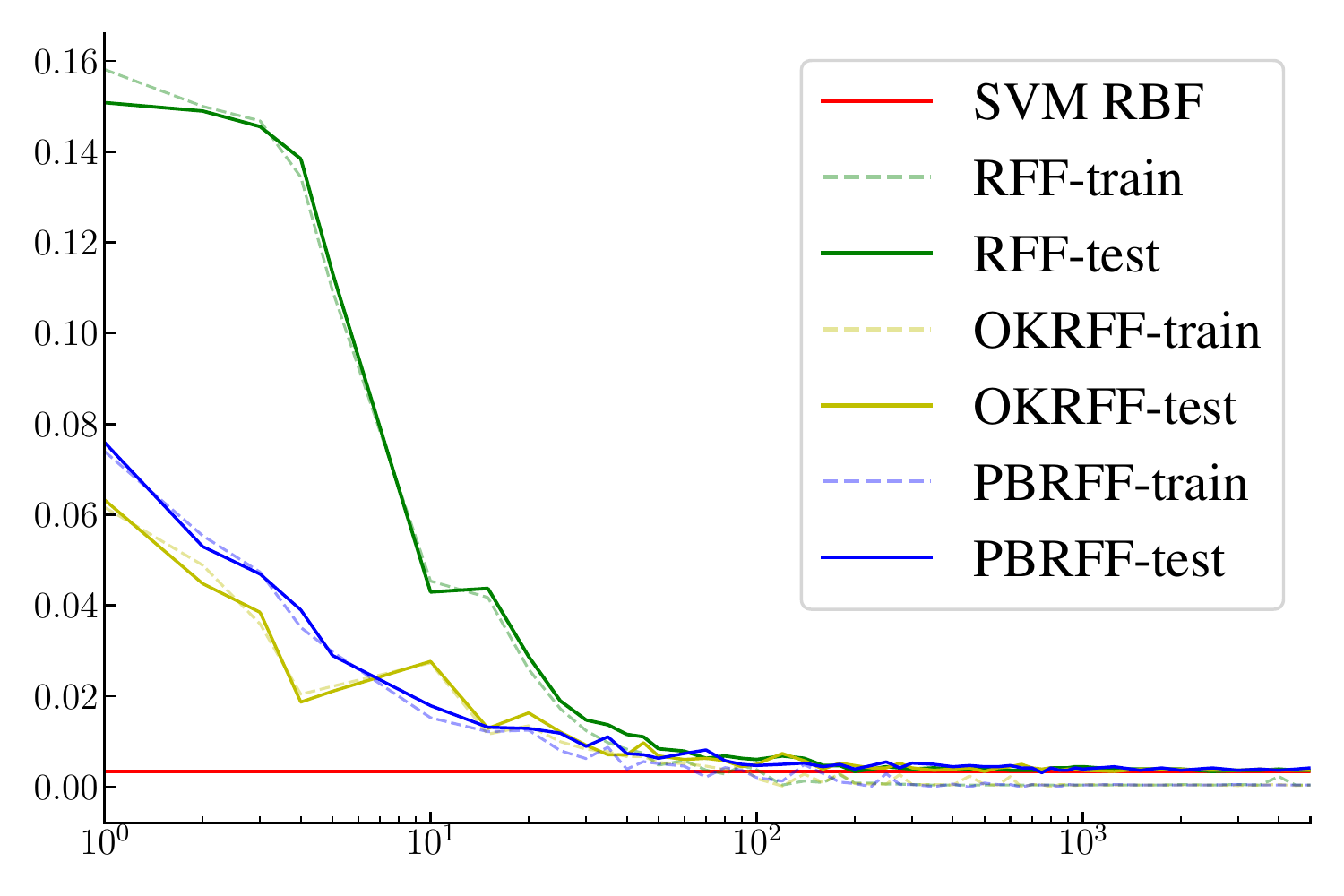}
        \caption{mnist17}
         \label{fig:err_vs_d_mnist17}
    \end{subfigure}
    \medskip
    \begin{subfigure}[t]{0.48\textwidth}
        \centering
        \includegraphics[trim={0cm 0cm 0cm 0cm},clip,width=\linewidth]{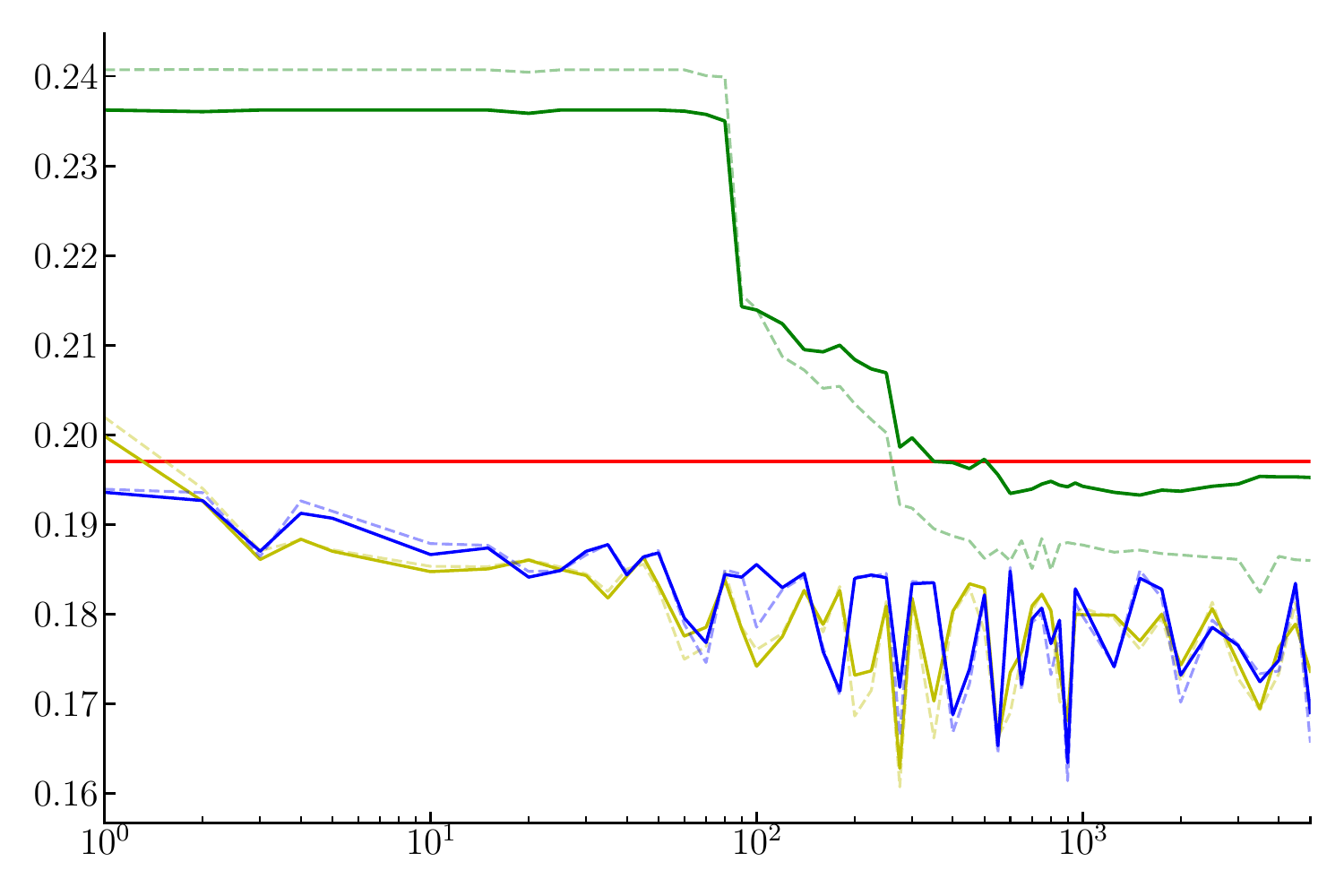}
        \caption{adult}
         \label{fig:err_vs_d_adult}
    \end{subfigure}
    \hfill
    \begin{subfigure}[t]{0.48\textwidth}
        \centering
        \includegraphics[trim={0cm 0cm 0cm 0cm},clip,width=\linewidth]{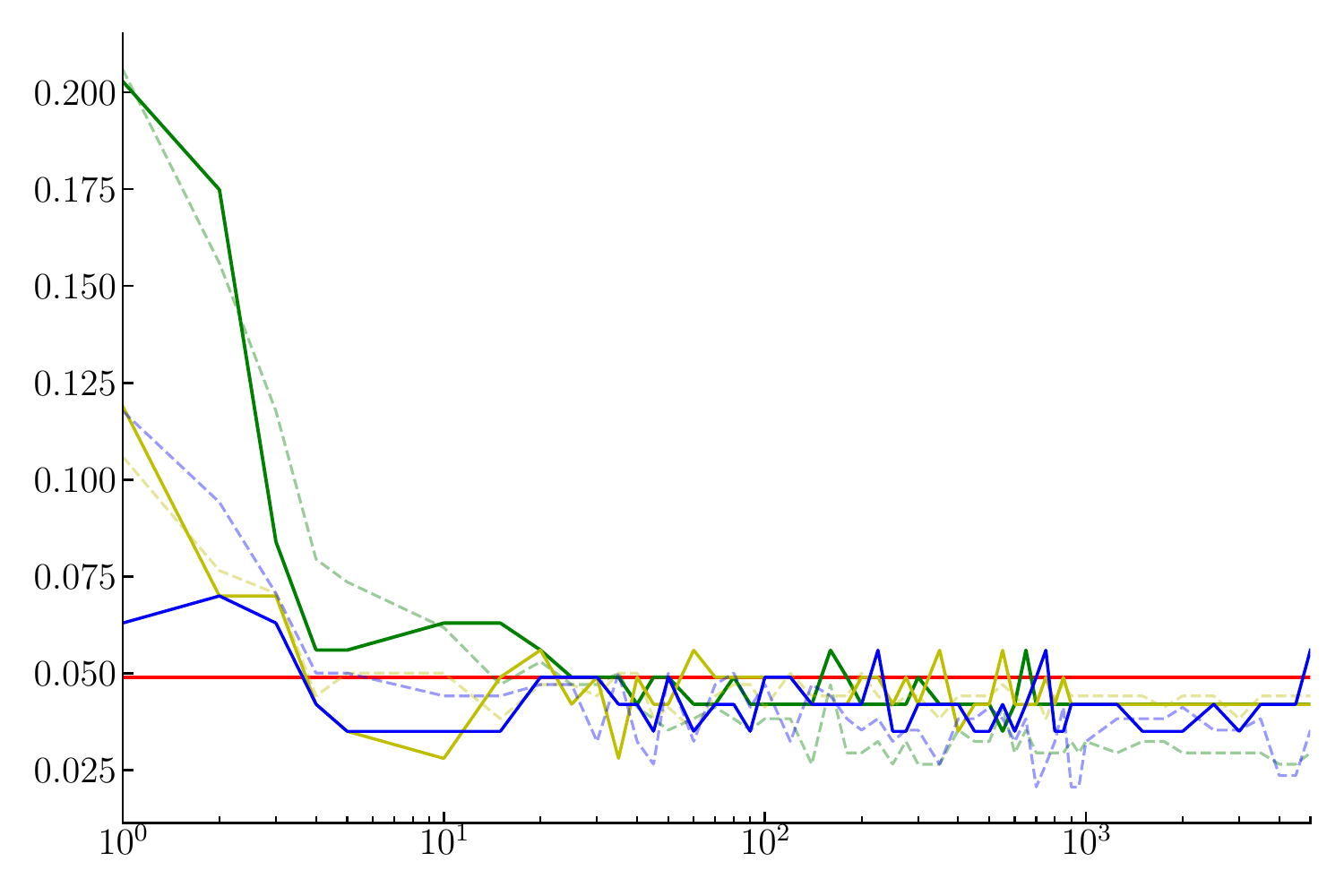}
        \caption{breast}
         \label{fig:err_vs_d_breast}
    \end{subfigure}
    \caption{Train and test error of the kernel learning approaches according to the number of random features $D$ on the remaining 4 datasets (not reported by Figure~\ref{fig:err_vs_d}).}
    \label{fig:err_vs_d_supp}
\end{figure*}

\begin{figure*}[t!]
    \centering
    \begin{subfigure}[t]{0.48\textwidth}
        \centering
        \includegraphics[trim={0cm 0cm 0cm 0cm},clip,width=\linewidth]{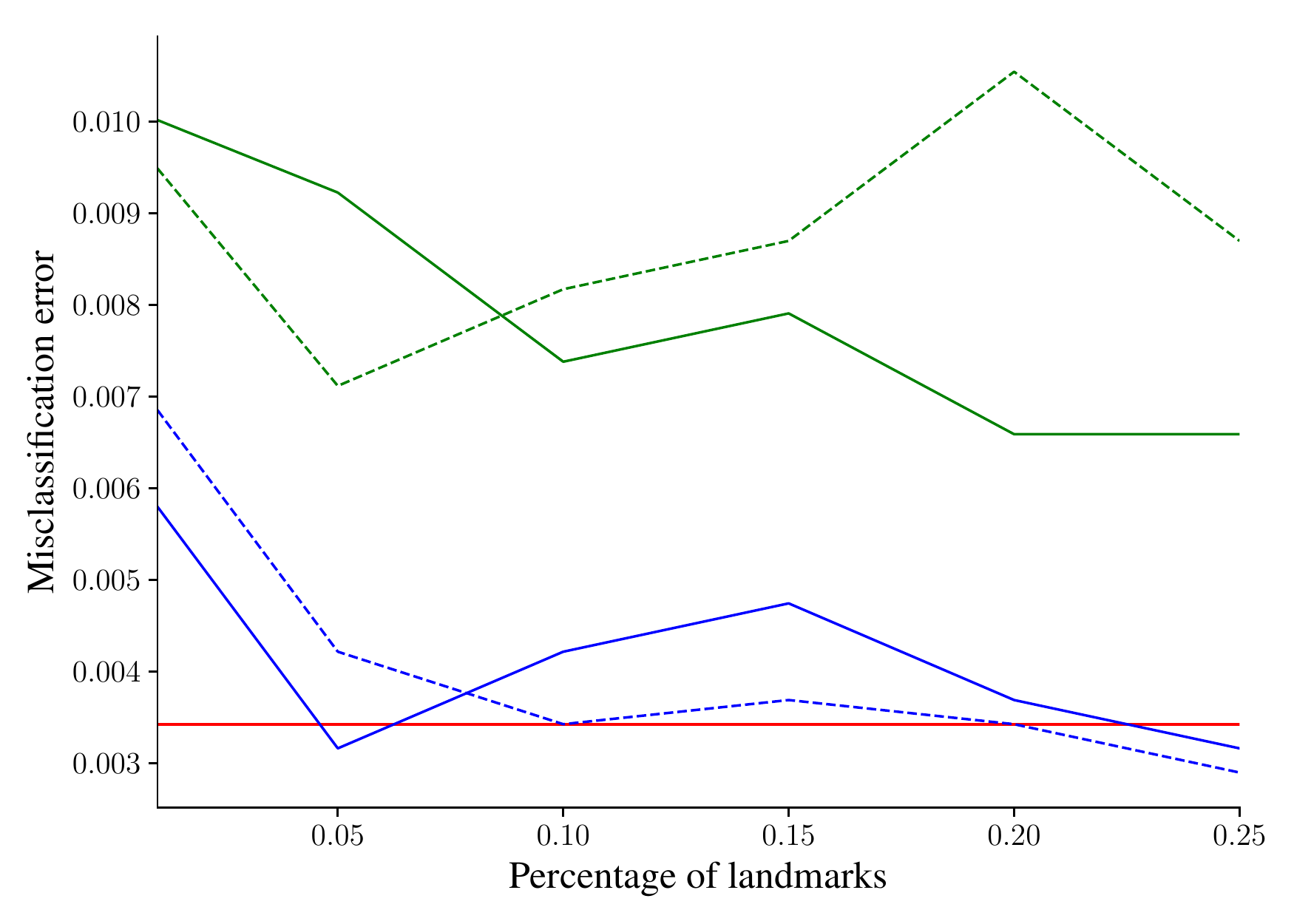}
        \caption{mnist17}
        \label{fig:err_vs_d_mnist17}
    \end{subfigure}
    \hfill
    \begin{subfigure}[t]{0.48\textwidth}
        \centering
        \includegraphics[trim={0cm 0cm 0cm 0cm},clip,width=\linewidth]{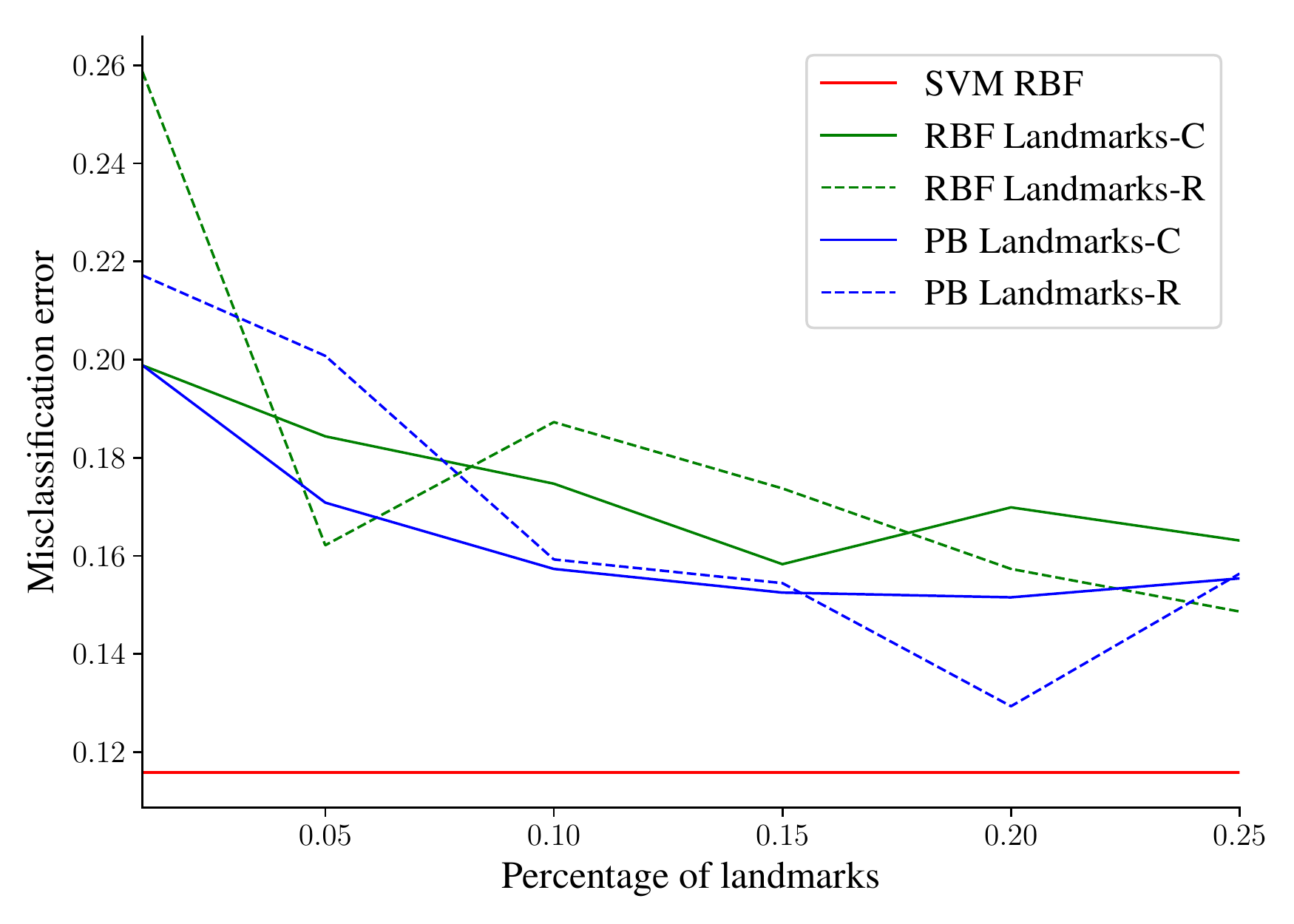}
        \caption{farm}
         \label{fig:err_land_farm}
    \end{subfigure}
    \medskip
    \begin{subfigure}[t]{0.48\textwidth}
        \centering
        \includegraphics[trim={0cm 0cm 0cm 0cm},clip,width=\linewidth]{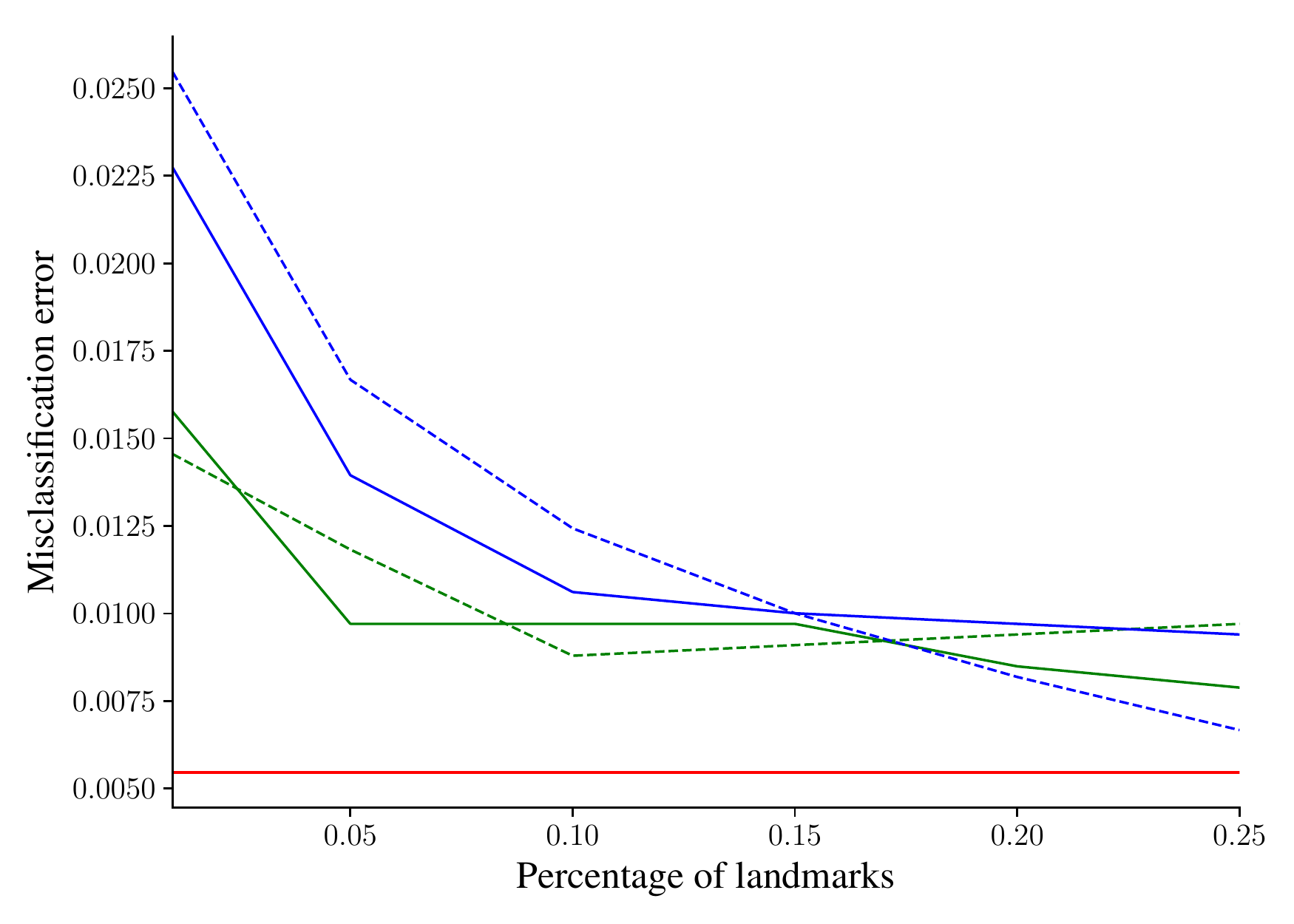}
        \caption{mnist56}
        \label{fig:err_land_mnist56}
    \end{subfigure}
    \hfill
    \begin{subfigure}[t]{0.48\textwidth}
        \centering
        \includegraphics[trim={0cm 0cm 0cm 0cm},clip,width=\linewidth]{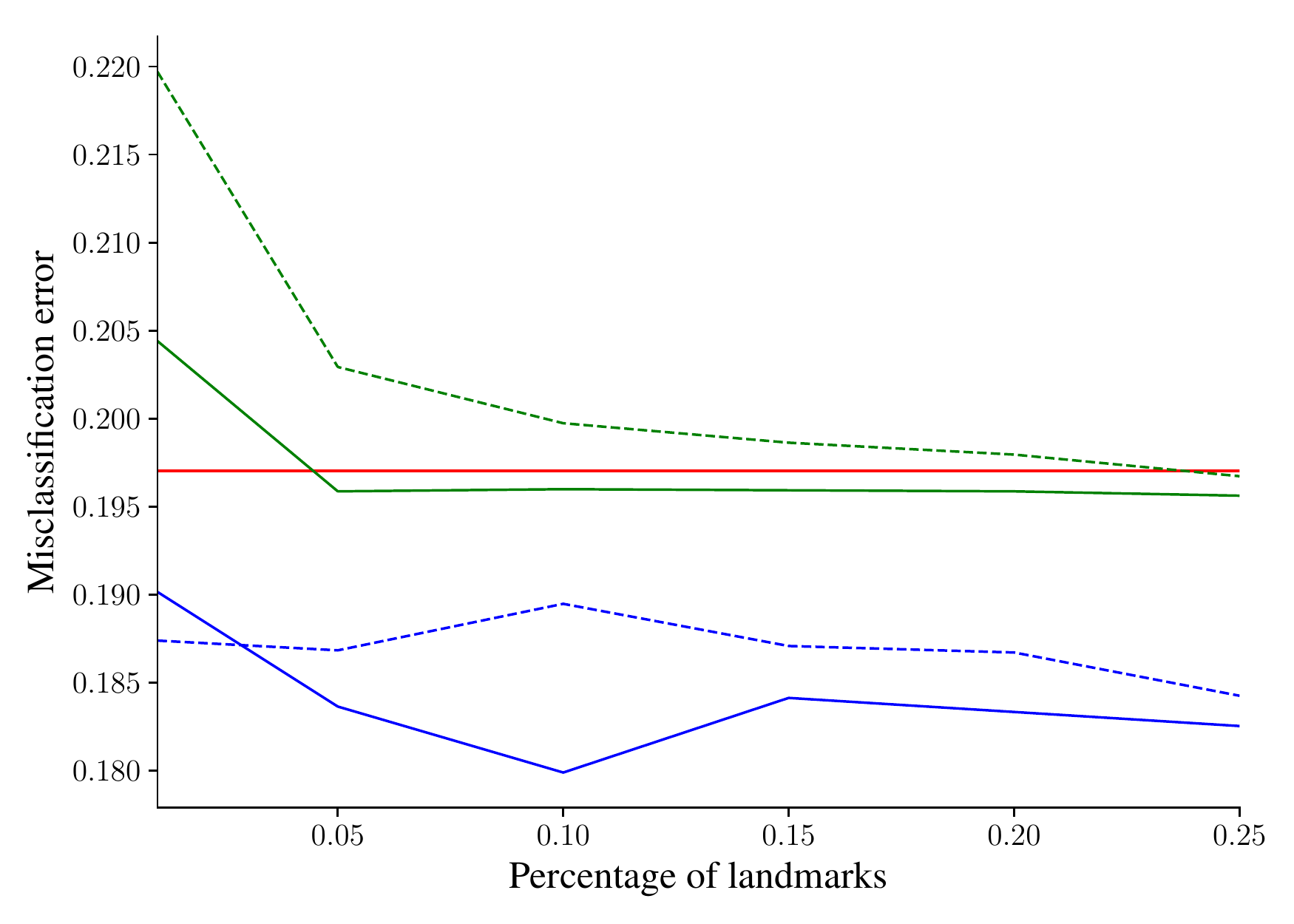}
        \caption{adult}
         \label{fig:err_land_adult}
    \end{subfigure}
    \medskip
    \begin{subfigure}[t]{0.48\textwidth}
        \centering
        \includegraphics[trim={0cm 0cm 0cm 0cm},clip,width=\linewidth]{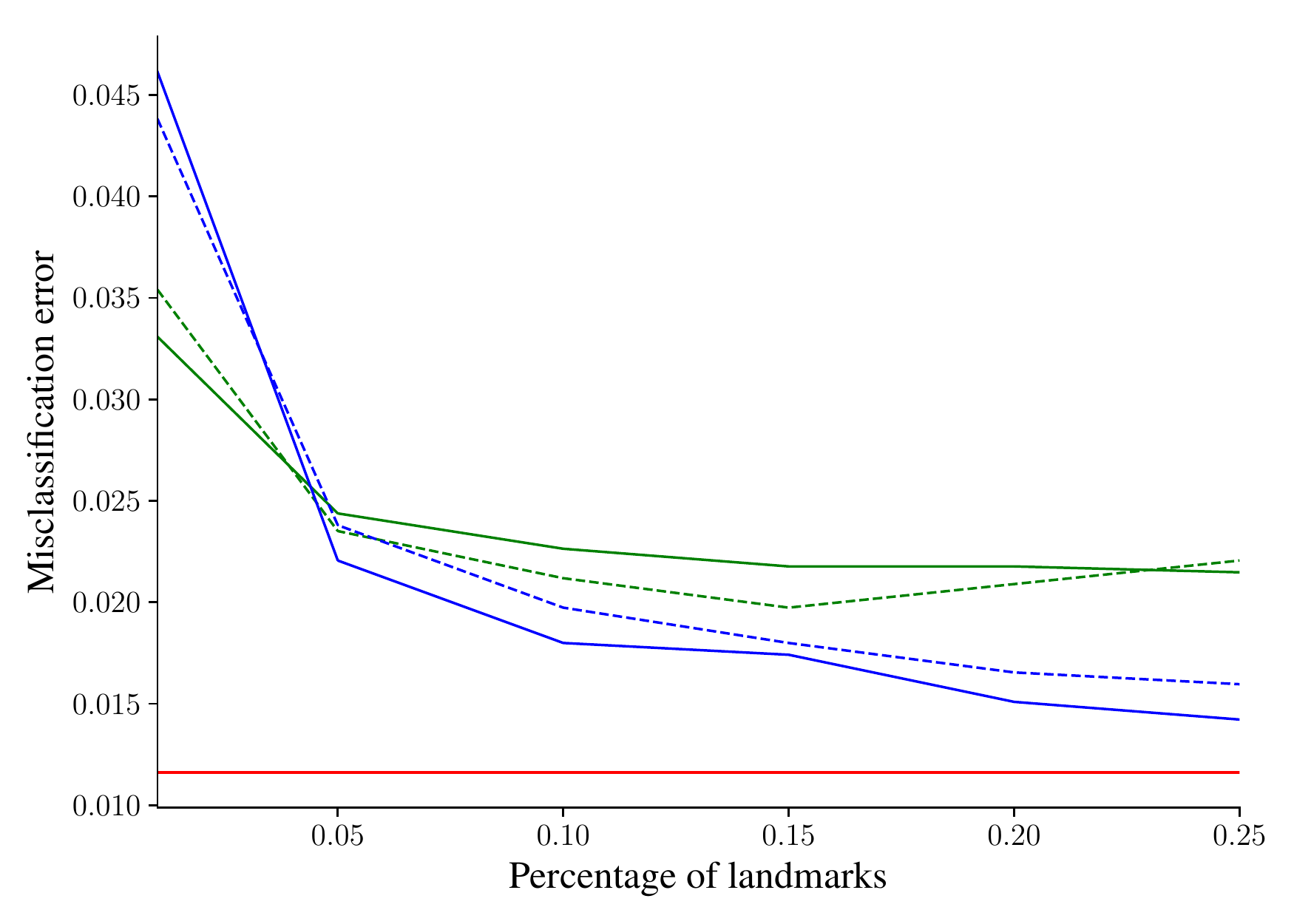}
        \caption{mnist49}
         \label{fig:err_land_mnist49}
    \end{subfigure}
    \hfill
    \begin{subfigure}[t]{0.48\textwidth}
        \centering
        \includegraphics[trim={0cm 0cm 0cm 0cm},clip,width=\linewidth]{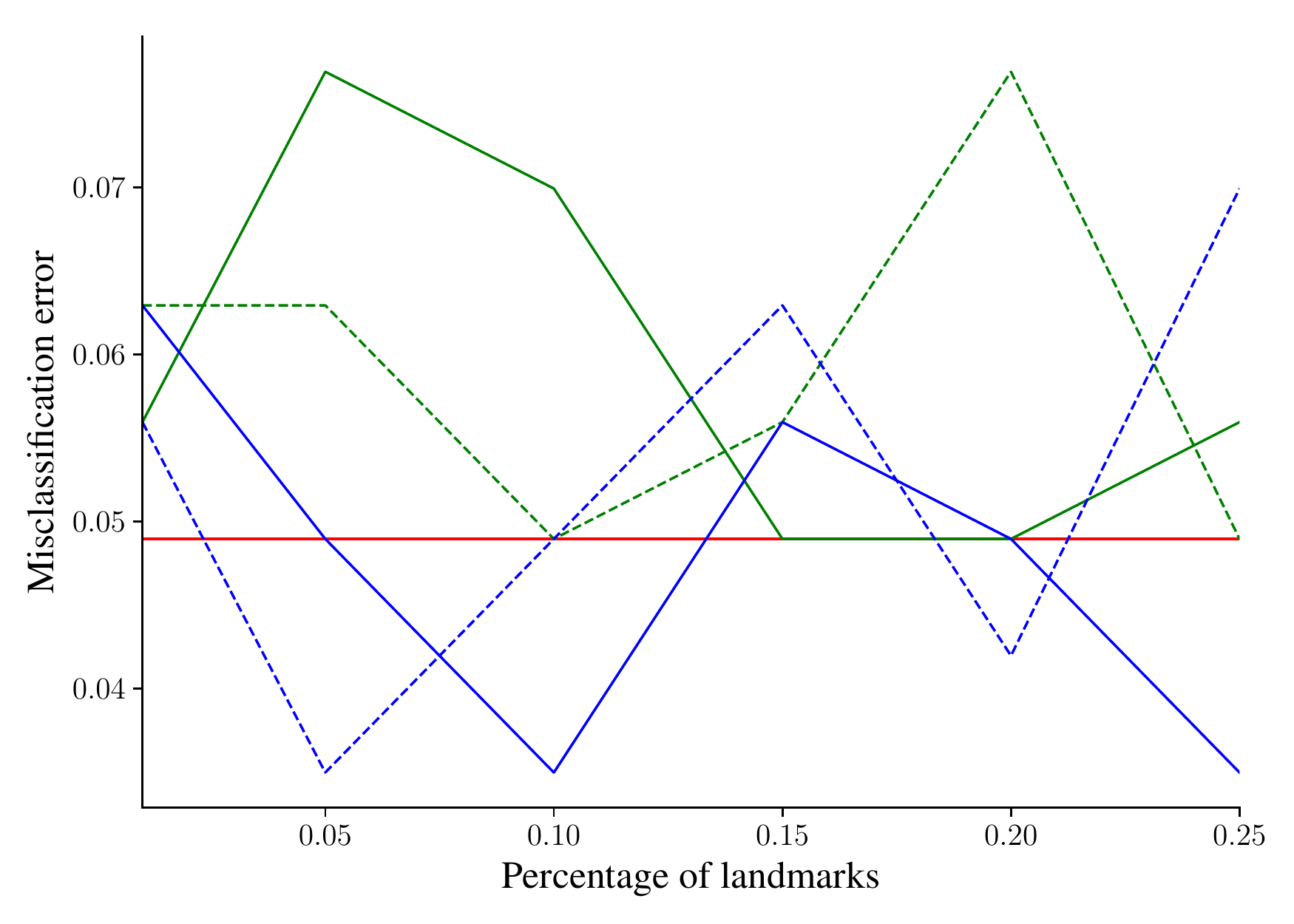}
        \caption{breast}
         \label{fig:err_land_breast}
    \end{subfigure}
    \caption{Behavior of the landmarks-based approach according to the percentage of training points selected as landmarks on the remaining 6 datasets (not reported by Figure~\ref{fig:error_landmarks_ads}).}
    \label{fig:err_land_supp}
\end{figure*}

\end{document}